\documentclass{article}

\PassOptionsToPackage{numbers, sort&compress}{natbib}

\usepackage{natbib}

\usepackage[final]{neurips_2022}

\usepackage{graphicx,graphics,amsmath,amssymb,amsfonts,verbatim,bm,color,mathrsfs, enumerate,amsthm,dsfont, subcaption}
\usepackage[normalem]{ulem}
\usepackage[shortlabels]{enumitem}
\usepackage[breaklinks=true,colorlinks,citecolor=blue,linkcolor=blue,urlcolor=blue]{hyperref}

\usepackage{thmtools,thm-restate}

\usepackage[utf8]{inputenc} 
\usepackage[T1]{fontenc}    
\usepackage{hyperref}       
\usepackage{url}            
\usepackage{booktabs}       
\usepackage{amsfonts}       
\usepackage{nicefrac}       
\usepackage{microtype}      
\usepackage{xcolor}         

\usepackage{times}
\usepackage{amsmath, nccmath, amssymb}
\usepackage{amsthm}
\usepackage{geometry}
\usepackage{comment}
\usepackage{enumitem}
\usepackage{color}
\usepackage{graphicx}
\usepackage{bbold}
\usepackage{xfrac}
\usepackage{thmtools}
\usepackage[ruled]{algorithm2e}
\usepackage{algorithmic}
\usepackage{mathtools}

\newtheorem{example}{Example} 
\newtheorem{theorem}{Theorem}
\newtheorem{lemma}{Lemma}
\newtheorem{remark}{Remark}
\newtheorem{definition}{Definition}
\newtheorem{proposition}{Proposition}
\newtheorem{corollary}{Corollary}

\newcommand{\G}{{\cal G}}

\title{Learning Dynamical Systems via Koopman Operator Regression in Reproducing Kernel Hilbert Spaces}

\DeclareMathOperator*{\argmin}{\ensuremath{\text{\rm arg\,min}}}

\DeclareMathOperator*{\range}{\ensuremath{\text{\rm Im}}}

\DeclareMathOperator*{\cl}{\ensuremath{\text{\rm cl}}}

\DeclareMathOperator{\Ker}{\ensuremath{\text{\rm Ker}}}

\DeclareMathOperator{\tr}{\ensuremath{\text{\rm tr}}}

\DeclareMathOperator*{\rank}{\ensuremath{\text{\rm rank}}}
\DeclareMathOperator*{\Span}{\ensuremath{\text{\rm span}}}

\DeclareMathOperator*{\Spec}{\ensuremath{\text{\rm Sp}}}
\DeclareMathOperator*{\sep}{\ensuremath{\text{\rm sep}}}

\providecommand{\norm}[1]{\lVert#1\rVert}
\providecommand{\abs}[1]{\lvert#1\rvert}
\newcommand{\scalarp}[1]{{\langle #1\rangle}_{\RKHS}}
\newcommand{\R}{\mathbb R}
\newcommand{\C}{\mathbb C}
\newcommand{\N}{\mathbb N}

\newcommand{\EE}{\ensuremath{\mathbb E}}
\newcommand{\PP}{\ensuremath{\mathbb P}}

\newcommand{\Id}{I}
\newcommand{\dynmap}{F}     
\newcommand{\Data}{\mathcal{D}}
\newcommand{\Koop}{A_{\im}}  
\newcommand{\Koopt}[1]{A^{#1}_{\im}}
\newcommand{\CME}{g_p}
\newcommand{\HKoop}{G_{\RKHS}}  
\newcommand{\RKoop}{G_\reg}  
 
\newcommand{\EEstim}{\widehat{G}}  
\newcommand{\Estim}{G}  
\newcommand{\Kx }{K} 
\newcommand{\Ky}{L} 
\newcommand{\Kxy}{M^\top} 
\newcommand{\Kyx}{M}
\newcommand{\Cx}{C } 
\newcommand{\Cy}{D} 
\newcommand{\Cxy}{T}  
\newcommand{\Cyx}{T^*}  
\newcommand{\ECx}{\widehat{C} } 
\newcommand{\ECy}{\widehat{D}} 
\newcommand{\ECxy}{\widehat{T} }  
\newcommand{\TZ}{Z}  
\newcommand{\EZ}{\widehat{Z}} 
\newcommand{\TS}{S}  
\newcommand{\ES}{\widehat{S}} 

\newcommand{\X}{\mathcal{X}} 
\newcommand{\F}{\mathcal{F}} 
\newcommand{\Risk}{\mathcal{R}} 
\newcommand{\IrRisk}{\mathcal{R}_{0}} 
\newcommand{\ExRisk}{\mathcal{E}} 
\newcommand{\RKHS}{\mathcal{H}} 
\newcommand{\Lii}{L^2_\im(\X)} 
\newcommand{\sigalg}{\Sigma_{\X}} 
\newcommand{\im}{\pi} 
\newcommand{\HS}[1]{{\rm{HS}}\left(#1\right)} 
\newcommand{\HSr}{{\rm{HS}}_r({\RKHS})}
\newcommand{\hnorm}[1]{\norm{#1}_{\rm{HS}}}

\newcommand{\noise}{\omega}
\newcommand{\noisedistribution}{\Omega}
\newcommand{\transitionkernel}{p} 
\newcommand{\transferop}{P} 
\newcommand{\reg}{\gamma}

\newcommand{\levec}{\widetilde{u}}
\newcommand{\revec}{\widetilde{v}}
\newcommand{\refun}{\psi}
\newcommand{\lefun}{\xi}

\newcommand{\levecs}{\widetilde{U}}
\newcommand{\revecs}{\widetilde{V}}

\newcommand{\bigO}{{\cal O}}
\usepackage[textwidth=2.0cm, textsize=tiny]{todonotes} 

\author{Vladimir R. Kostic\thanks{Equal contribution, corresponding authors.}\\
      Istituto Italiano di Tecnologia\\
      University of Novi Sad\\
      \texttt{vladimir.kostic@iit.it}
      \And
      Pietro Novelli \footnotemark[1]\\
      Istituto Italiano di Tecnologia\\
      \texttt{pietro.novelli@iit.it}
      \And
      Andreas Maurer \\
      Istituto Italiano di Tecnologia\\
      \texttt{am@andreas-maurer.eu}
      \And
      Carlo Ciliberto \\
      University College London\\
      \texttt{c.ciliberto@ucl.ac.uk}
      \And
      Lorenzo Rosasco \\
      University of Genova \\
      Massachusetts Institute of Technology \\
      Istituto Italiano di Tecnologia\\
      \texttt{lrosasco@mit.edu}
	  \And
      Massimiliano Pontil \\
      Istituto Italiano di Tecnologia\\
      University College London\\
      \texttt{massimiliano.pontil@iit.it}
      }

\begin{document}

\maketitle
\begin{abstract}
We study a class of dynamical systems modelled as Markov chains that admit an invariant distribution via the corresponding transfer, or Koopman, operator. While data-driven algorithms to reconstruct such operators are well known, their relationship with statistical learning is largely unexplored. We formalize a framework to learn the Koopman operator from finite data trajectories of the dynamical system. We consider the restriction of this operator to a reproducing kernel Hilbert space and introduce a notion of risk, from which different estimators naturally arise. We link the risk with the estimation of the spectral decomposition of the Koopman operator.
These observations motivate a reduced-rank operator regression (RRR) estimator. We derive learning bounds for the proposed estimator, holding both in i.i.d. and non i.i.d. settings, the latter in terms of mixing coefficients. Our results suggest RRR might be beneficial over  other  widely used estimators as confirmed in numerical experiments  both for  forecasting and mode decomposition.

\end{abstract}

\section{Introduction}
Dynamical systems \cite{Lasota1994, Meyn1993} provide a framework to study a variety of 
complex phenomena in science and engineering. 
For instance, they find wide applications in diverse fields such as finance~\cite{Pascucci2011}, 
robotics \cite{Folkestad2021, Bruder2021}, atomistic simulations \cite{Schutte2001, Mardt2018, McCarty2017}, open quantum system dynamics~\cite{Lindblad1976, Gorini1976}, and many more. 
Because of their practical importance, research around dynamical systems is and has been abundant, see e.g. \cite{fuchs2012,strogatz2014} and references therein. 

In light of recent machine learning progress, it  is appealing to ask if the properties of dynamical systems can be
 estimated ({\em learned}) from empirical data. Beyond machine learning this question has a long history in dynamical systems \cite{Brunton2022}. The go-to reference for data-driven algorithms to reconstruct dynamical systems is~\cite{Kutz2016}, where numerous methods based on so-called dynamic mode decomposition (DMD)
are discussed along with interesting applications.  The literature  on  various  theoretical aspects of dynamical systems  is also rich \cite{Meyn1993,Mauroy2020}. Our starting observation is that although data-driven algorithms to reconstruct dynamical systems are well known, their relationship with statistical learning \cite{vapnik1998} is largely unexplored. Our broad goal is to build a tie between these two important areas of research and to establish firm theoretical grounds for  data driven approaches, to derive statistical guarantees and a 
foundation in which learning dynamical systems can be tackled in great generality.


In this paper, we present a framework for {\it learning} dynamical systems from data obtained from  one or multiple trajectories. The focus is  both {\em predicting} the future states of the system and  {\em interpreting} the underlying dynamic. 
 The initial observation is the fact that,  under suitable assumptions, a dynamical system can be completely characterized by a {\it linear operator}, known  as Koopman (or transfer) operator~\cite{Budisic2012, Mauroy2020}. More precisely, the Koopman operator describes how functions (observables) of the state of the system evolve over time along its trajectories. Further, the spectral decomposition of the  Koopman operator, along with the mode decomposition,  allows us to interpret the dynamical and spatial properties of the system~\cite{Rowley2009, Mezic2019, Crnjaric2019}.
 In view of these results,  learning a dynamical system can be cast as the problem of learning the corresponding Koopman operator and associated mode decomposition. 
 

A key insight in our approach is to consider 
the restriction of Koopman operators to reproducing kernel Hilbert spaces. With this choice,  Hilbert-Schmidt operators become the natural hypothesis space and kernel methods can be exploited \cite{Klus2019}. We further link the proposed framework to conditional mean embeddings~\cite{Klus2019,grune2012,Muandet2017}. This allows us to formalize the estimation of the Koopman operator as a risk minimization problem and derive a number of estimators  as instances of classical empirical risk minimization under different constraints. We dub the problem Koopman operator regression. 
In our framework DMD and some of its variants~\cite{Kutz2016} are recovered as special cases. Moreover, our analysis highlights the importance of rank constrained estimators, and, following this observation, we introduce and analyze an estimator akin to reduced rank regression (RRR)~\cite{Izenman}. 
Within our  statistical learning framework the learning properties of the studied estimators can be characterized in terms of non asymptotic error bounds derived from concentration of measure results  for mixing processes. Theoretical results are complemented by numerical experiments where  we investigate the properties the estimators, and  show  they can be smoothly interfaced with deep learning techniques.  We note that, both kernel methods~\cite{OWilliams2015,Kawahara2016,Bouvrie2017,Klus2019,Das2020,Alexander2020} and deep learning approaches ~\cite{Lusch2018,Fan2021,Bevanda2021}, have been recently considered to learn Koopman operators. Compared to these works (notably~\cite{Klus2019}, whose setting is 
closely related to ours), we provide a statistical learning framework connecting to the classical notions of risk, which we further 
link 
to the estimation of the spectrum of the Koopman operator. Moreover, we derive non asymptotic and non-i.i.d. learning bounds and introduce and study a novel constrained rank estimator RRR. 
 
{\bf Contributions.} In summary our main contributions are:~{\bf 1)} We present a statistical learning framework for Koopman operator regression; {\bf 2)} We bound the error in estimating the Koopman mode decomposition and its eigenvalues by the risk of an estimator (Theorem~\ref{thm:KMD_learning}); {\bf 3)} We present a novel reduced-rank estimator and show that it can be computed and used efficiently (Theorem \ref{thm:DMD_theorem}); {\bf 4)} We provide a statistical risk bound supporting the proposed estimator (Theorem~\ref{thm:UB_main_text}) and introduce a new tool (Lemma~\ref{lem:blockprocess}) which is key in extending the bound to the non-i.i.d. setting. 

{\bf Notation.} For any non-negative integers $n,m$ with $n>m$ we use the notation $[m{:}n]=\{m,\dots,n\}$ and $[n]=[1{:}n]$. $\Lii := L^{2}(\X, \im)$ is the space of real valued functions on  $\X$, that are square-integrable with respect to $\im$.  Given two separable Hilbert spaces $\mathcal{H}$ and $\mathcal{G}$, we let $\HS{\mathcal{H},\mathcal{G}}$ be the Hilbert space of Hilbert-Schmidt (HS) operators from $\mathcal{H}$ to $\mathcal{G}$ endowed with the norm  $\hnorm{A}^{2} \equiv\sum_{i \in \mathbb{N}} \Vert Ae_{i} \Vert_{\mathcal{G}}^{2}$, for $A \in \HS{\mathcal{H},\mathcal{G}}$, where $(e_{i})_{i\in \mathbb{N}}$ is an orthonormal basis of $\mathcal{H}$. We use the convention $\HS{\mathcal{H}} {=} \HS{\mathcal{H},\mathcal{H}}$. 
The standard norms in Hilbert spaces and operator norms are denoted by $\norm{\cdot}$, where the space is clear from the context. Given an operator $A \in \HS{\mathcal{H}}$, we denote with $[\![A]\!]_r$ its $r$-truncated singular value decomposition and its $i$-th singualr value by $\sigma_i(A)$.

\section{Background on Koopman operator theory}\label{sec:koopman_theory}

We briefly recall the basic notions related to  Markov chains and  Koopman operators and refer to App.~\ref{app:background} and~\cite{Lasota1994, Meyn1993, Mauroy2020}  for further details.

Let  $\bm{X} := \left\{X_{t} \colon t\in \N \right\}$ be a family of random variables with values in a measurable space $(\X, \sigalg)$, called state space. We call $\bm{X}$  a  {\em Markov chain} if $\PP\{ X_{t+1} \in B \,\vert\, X_{[t]} \} = \PP\{X_{t + 1} \in B \,\vert\, X_t \}$. 
Further, we call $\bm{X}$ {\em time-homogeneous} if there exists   $\transitionkernel\colon \X \times \sigalg \to [0,1]$, called {\it transition kernel}, such that,  for every $(x, B) \in \X \times \sigalg$ and every $t \in \N$,
\[
\mathbb{P}\left\{X_{t + 1} \in B \,\middle| \,X_{t} = x \right\} = \transitionkernel (x,B).
\]

In this work we consider only discrete Markov chains with $t \in \N$, but we note that any continuous Markov process with $t \in \R$ can be reduced to a discrete chain by sampling it at times $t_{n} = n\Delta t$ with $n \in \N$ and $\Delta t$ fixed. For an alternative approach to approximate continuous dynamics see e.g.~\cite{Rosenfeld2021}.

For a set $\F$  of real valued and measurable functions on $\X$, the \textit{Markov transfer operator}  $A_{\F} \colon \F\to \F$ is defined as 
\begin{equation}\label{eq:Koopman_F}
	A_{\F} f(x) := \int_{\X} p(x, dy)f(y) = \mathbb{E}\left[f(X_{t + 1}) \,\middle |\, X_{t} = x\right], \quad f\in\F,\,x\in\X.
\end{equation}
A possible choice is  $\F=L^\infty(\X)$,   the space of bounded functions on $\X$~\cite{Lasota1994}. 
We are interested in another common choice related to the existence of an \textit{invariant measure} $\im$ satisfying 
$\pi(B) {=} \int_{\X} \pi(dx)p(x,B),~ B\in\sigalg.$
In this case, it is possible to take $\F=\Lii$, and  easy to see that $\|A_{\F}\|\le 1$, that is the Markov transfer operator is a bounded linear operator. In the following, we denote by $\Koop$ the Markov transfer operator on $\Lii$, and always assume the existence of an invariant measure. We note that, its existence can be proven for large classes of Markov chains, see e.g.~\cite{prato1996}.  Also, to derive the statistical bounds in Sec.~\ref{sec:bounds} we  assume that the Markov chain is mixing~\cite{Lasota1994}.

\begin{example}\label{ex:noisy_ds}
An important example of the above  construction is given by discrete dynamical systems with additive noise. That is, given a state space $\X\subseteq\R^d$, a mapping $F\colon \X\hspace{.05truecm}{\to}\hspace{.05truecm}\X$ and a probability distribution $\noisedistribution$ on $\X$ we let 
\(
X_{t+1} = F(X_t)+\noise_t,~~t\in\N, 
\)
where 
$\noise_t$ are i.i.d. zero mean random variables with law $\noisedistribution$. 
The corresponding transition kernel is $p(x,B)= \noisedistribution(B{-}F(x))$, for which the existence of an invariant measure is ensured e.g. when $\noisedistribution$ is absolutely continuous with respect to the Lebesgue measure and its density is strictly positive (see Remark 10.5.4 in~\cite{Lasota1994}).
\end{example}

\begin{remark}\label{rem:reversible_dynamics}
Whenever $\im(dx)p(x,dy) = \im(dy)\transitionkernel(y,dx)$ the Markov chain is said to be {\em reversible}. In this case it readily follows that the Koopman operator is self-adjoint $\Koop = \Koopt{*}$. In statistical physics the reversibility condition is also called {\em detailed balance} and is linked to the symmetry with respect to time reversal. Since a large amount of microscopical equations of motion in both classical and quantum physics are time-reversal invariant, learning self-adjoint Koopman operators is of paramount importance in the field of machine learning for physical sciences.
\end{remark}

{\bf Koopman Operator and Mode Decomposition.~} In dynamical systems, $A_\F$ is known as the (stochastic) \textit{Koopman operator} on the space of  observables $\F$. An important fact is that its linearity can be exploited to 
compute a spectral decomposition.  Indeed,  in many situations, and notably for compact Koopman operators, there exist scalars $\lambda_i\in\C$, and observables  $\refun_i\in\Lii$ satisfying the eigenvalue equation $\Koop\refun_i \hspace{.05truecm}{=}\hspace{.105truecm} \lambda_i\refun_i$. Leveraging the eigenvalue decomposition, 
the dynamical system can be decomposed
into superposition of simpler signals that can be used in different tasks such as system identification and control, see e.g. \cite{Brunton2022}.
More precisely, given an observable $f\in\Span\{\refun_i\,\vert\,i\in\N\}$ there exist corresponding scalars $\gamma_i^f\in\C$ known as Koopman modes of $f$, such that
\begin{equation}\label{eq:koopman_MD}
\Koopt{t} f(x) = \EE[f(X_t)\,\vert\, X_0 = x] =  \sum_{i\in\N}\lambda_i^t \gamma^f_i \refun_i(x), \quad x\in\X,\,t\in\N.
\end{equation}
This formula is known as \textit{Koopman Mode Decomposition} (KMD) \cite{Budisic2012,AM2017}. It decomposes the expected dynamics observed by $f$ into \textit{stationary} modes $\gamma_i^f$ that are combined with \textit{temporal changes} governed by eigenvalues $\lambda_i$ and \textit{spatial changes}  
governed by the eigenfunctions $\refun_i$. 
We notice however that the Koopman operator, in general, is not a normal compact operator, hence its eigenfunctions may not form a complete orthonormal basis of the space which makes learning KMD challenging.


In many practical scenarios the transition kernel $p$, hence $\Koop$, is unknown, but data from one or multiple system trajectories are available. We are then interested into learning the Koopman operator, and corresponding mode decomposition, from the data. Next we discuss how to accomplish this task with the aid of kernel methods.

\section{Statistical Learning Framework}
\label{sec:learning}
In this section we choose the space of observables $\F$ to be a reproducing kernel Hilbert space (RKHS) and present a framework for learning the Koopman operator on $\Lii$ restricted to this space and the associated Koopman mode decomposition.


{\bf Learning Koopman Operators.~} 
Let $\RKHS$ be an RKHS with kernel $k:\X \times \X \rightarrow \mathbb{R}$ \cite{aron1950} and let $\phi :\X \to \RKHS$ be an associated feature map, such that $k(x,y) = \scalarp{\phi(x),\phi(y)}$ for all $x,y \in \X$.  
We assume that that $k(x,x) < \infty$, $\pi$-almost surely.  
This ensures that $\RKHS \subseteq \Lii$ and the injection operator $\TS_{\im} \colon \RKHS \to \Lii$ given by $(\TS_{\im} f)(x)=f(x)$, $x\in\X$
is a well defined Hilbert-Schmidt operator~\cite{caponnetto2007,Steinwart2008}. 
Then, the Koopman operator restricted to $\RKHS$ is given by 
\[
\TZ_\im := \Koop\TS_\im \colon \RKHS\to\Lii.
\]
Note that  unlike $\Koop$,  $\TZ_\im$ is Hilbert-Schmidt since $\TS_\im$ is so. 
It is then natural to approximate $\TZ_\im$ by means of Hilbert-Schmidt operators. 
More precisely, for $\Estim\in\HS{\RKHS}$ we approximate  $\TZ_\im$ by $ \TS_\im \Estim$, and 
measure the corresponding error as 
$\hnorm{\TZ_\im {-} \TS_\im \Estim}^2$. To that end, given an an orthonormal basis $(h_i)_{i\in\N}$ of $\RKHS$, we introduce the risk
\begin{equation}
    \label{eq:true_risk}
\Risk(\Estim):= \sum_{i \in \N}  \EE_{x\sim \im} \EE \big[ \left[h_i(X_{t+1}) - (\Estim h_i)(X_t)\right]^2\,\vert X_t  = x\big]
\end{equation}
as the cumulative expected one-step-ahead prediction error over {\it all} observables in  $\RKHS$. One can show (see Prop.~\eqref{prop:risk} in App.~\ref{app:learning}) that such risk can be decomposed as $\Risk(\Estim)= \IrRisk +\ExRisk(\Estim)$, where
\begin{equation}
    \label{eq:ex_ir_risk}
\IrRisk:=\hnorm{\TS_\im}^2-\hnorm{\TZ_\im}^2\geq0\;\text{ and }\; \ExRisk(\Estim)=\hnorm{\TZ_\im-\TS_\im\Estim}^2,
\end{equation}
are the {\em irreducible risk} and the {\em excess risk}, respectively. 
As clear from the above discussion, the Koopman operator and corresponding risk  are typically not available in practice and what is available is  a dataset of observations $\Data := (x_i,y_i)_{i=1}^n \in (\X \times \X)^n$. Here, $x_i$ and $y_i$ are two consecutive observations of the state of the system. In classical statistical learning,  the data  is assumed sampled i.i.d. from the joint probability measure $\rho(dx,dy):=\im(dx)\transitionkernel(x,dy)$. In the case of dynamical systems, it is natural to assume the the data are obtained by sampling a trajectory  $y_i=x_{i+1}$, for $i \in [n-1]$. Then, the problem of learning $\Koop$ on a RKHS, named here Koopman operator regression,
reduces to:
\begin{equation}\label{eq:KRP}
\text{ Given the data }\; \Data, \;\text{ solve }\min_{\Estim\in\HS{\RKHS}} \Risk(\Estim).
\end{equation}
As discussed in Sec.~\ref{sec:koopman_theory}, a central idea associated to Koopman operators is the corresponding mode decomposition. It is then natural to ask whether an approximate mode decomposition can be derived from a Koopman estimator. The following proposition provides a useful step in this direction. 
Here and in the rest of the paper by $\cl(\cdot)$ we denote the closure of a subspace of the Hilbert space, and we say that a finite rank operator $\Estim\in\HS{{\cal H}}$ is {\em non-defective} if and only if its matrix representation is non-defective, i.e. (not necessarily unitarily) diagonalizable.

\begin{restatable}{proposition}{propRisk}\label{prop:hypothesis_space}
If  $\range(\TZ_\im)\subseteq\cl(\range(\TS_\im))$, 
then for every $\delta >0$ there exists a finite rank non-defective operator $\Estim \in\HS{\RKHS}$ such that $\ExRisk(\Estim)<\delta$.
\end{restatable}
The above proposition shows that if the RKHS $\RKHS$ is, up to its closure in $\Lii$, an invariant subspace of the Koopman operator $\Koop$, then finite rank non-defective HS operators on $\RKHS$ approximate arbitrarily well the restriction of $\Koop$ onto $\RKHS$. In particular, this is always true for a wide class of kernels, called universal kernels \cite{Steinwart2008}, for which $\RKHS$ is dense in $\Lii$, i.e. $\cl(\range(\TS_\im))=\Lii$.

\begin{remark}\label{rem:well_mis_spec}
Since in the above proposition $\inf_{G\in\HS{\RKHS}}\ExRisk(G)=0$, we can distinguish between two cases depending on whether the infimum is attained or not. In the former case, known in the literature as \textit{well-specified case}, there exists $\HKoop \in \HS{\RKHS}$ such that  $\TZ_\im = \TS_\im \HKoop$, which implies that $\HKoop\colon\RKHS\to\RKHS$ defines $\im$-a.e. the Koopman operator on the observable space $\RKHS$, i.e. $\HKoop f = \EE[f(X_{t+1})\,\vert\,X_t = \cdot]$ $\im$-a.e. for every $f\in\RKHS$. In the latter case, known as \textit{misspecified case}, $\RKHS$ does not admit a Hilbert-Schmidt Koopman operator $\RKHS\to\RKHS$. 
\end{remark}

\begin{remark}\label{rem:self-adjoint}
In the misspecified case, a Koopman operator on $\RKHS$ might still exist as a bounded albeit not Hilbert-Schmidt operator. In App.~\ref{app:kor-cme} we show that for reversible Markov processes (see Rem.~\ref{rem:reversible_dynamics}), if  $\range(\TZ_\im)\subseteq\cl(\range(\TS_\im))$, there exists Koopman operator $\HKoop\colon\RKHS\to\RKHS$ such that $\norm{\HKoop}\leq1$.
\end{remark}



{\bf Learning  the Koopman mode decomposition.~} Techniques to estimate  the Koopman mode decomposition~\eqref{eq:koopman_MD} from data, are broadly referred to as   \textit{ Dynamic Mode Decomposition (DMD)} \cite{Kutz2016}. We next introduce a DMD approach following the discussion above. 
Let $r\in\N$ and a non-defective 
$\Estim\in\HSr:=\{\Estim \in\HS{\RKHS}\,\vert\,\rank(\Estim)\leq r\}$. Then, there exists a spectral decomposition of $\Estim$ given by $(\lambda_i,\lefun_i,\refun_i)_{i=1}^{r}$ where $\lambda_i\in \C$ and $\lefun_i$ and $\refun_i$ are complex-valued function with components in $\RKHS$, 
such that $\Estim=\sum_{i=1}^{r}\lambda_i\refun_i \otimes \overline{\lefun}_i$, where $\Estim\refun_i=\lambda_i\refun_i$, $\Estim^*\lefun_i=\overline{\lambda}_i\lefun_i$ and $\scalarp{\refun_i, \overline{\lefun}_j}=\delta_{ij}$, where  $\delta_{ij}$ is Kronecker delta symbol, $i,j\in[r]$. This implies, for any $f\in\RKHS$, that 
\begin{equation}\label{eq:DMD_exp}
\Estim^t f = \textstyle{\sum_{i\in[r]}}\lambda_i^t \gamma_i^f \refun_i,\quad t\geq1.
\end{equation} 
The coefficients $\gamma_i^f:=\scalarp{f,\overline{\lefun}_i}$, $i\in [r]$,  are called  \textit{dynamic modes} of the observable $f$ and  expression~\eqref{eq:DMD_exp} is known as the DMD corresponding to  $\Estim$.
Next, we upper bound the error in  estimating the mode decomposition of $\Koop$ (see Eq.~\eqref{eq:koopman_MD}) by the DMD of a non-defective operator $\Estim \in \HSr$. 
\begin{restatable}{theorem}{thmKMDLearning}\label{thm:KMD_learning}
Let $\Estim \in \HSr$ and $(\lambda_i,\lefun_i,\refun_i)_{i=1}^r$ its spectral decomposition. Then for every $f\in\RKHS$ 

\begin{equation}\label{eq:KMD_bound}
\EE[f(X_{t})\,\vert\,X_0 = x] = \sum_{i\in[r]}\lambda_i^t \gamma_i^f \refun_i(x) +  \norm{\TZ_\im-\TS_\im\Estim}\, {\rm err}^f(x),\quad x\in\X, 
\end{equation}
where 
${\rm err}^f\in\Lii$, and $\norm{{\rm err}^f}\leq (t-1)\norm{\Estim f}+\norm{f}$, $t\geq1$. Moreover, for any  $i{\in} [r]$, 
\begin{equation}\label{eq:bound_efun}
\norm{\Koop\TS_\im \refun_i - \lambda_i \TS_\im \refun_i}\leq  \frac{\norm{\TZ_\im-\TS_\im\Estim}\,\norm{\Estim}}{\sigma_r(\TS_\im \Estim)} \norm{\TS_\im \refun_i}. 
\end{equation}
\end{restatable}

This theorem provides KMD approximation results for the DMD obtained from an estimator $G$ in $\HSr$. First, since $\norm{\TZ_\im-\TS_\im\Estim}\leq \sqrt{\ExRisk(\Estim)}$, equation~\eqref{eq:KMD_bound} shows that DMD for an estimate $\Estim \in \HSr$ incurs an error which is at most proportional to the (square root) of the corresponding excess risk $\ExRisk(\Estim)$. Further, the error degrades as $t$ increases. This implies that the prediction error ($t=1$) is only controlled by the risk  but forecasting ($t\geq1$) will get in general increasingly harder for larger $t$.
Second, inequality~\eqref{eq:bound_efun} shows that $(\lambda_i,\TS_\im\refun_i)$ is approximately an eigenpair of the Koopman operator $\Koop$. Indeed, it  guarantees that all the eigenfunctions of the estimator $G$, considered as an equivalence class in $\Lii$, approximately satisfy the Koopman eigenvalue equation. Inequality~\eqref{eq:bound_efun} provides a relative error bound  controlled by the excess risk  $\ExRisk(\Estim)$, where the approximation quality worsen as higher ranks are considered.
This provides additional motivation to study 
low rank estimators of $\Koop$, see  Sec. \ref{sec:erm}. Moreover, we point out that the bounds in Thm.~\ref{thm:KMD_learning} are tight, since on any RKHS $\RKHS$ spanned by a finite-number of Koopman eigenfunctions it exists a finite rank $\HKoop$ yielding $\norm{\TZ_{\im} - \TS_{\im}\HKoop}_{{\rm HS}} = 0$. For additional discussions see Exm.~\ref{ex:OU} of App.~\ref{app:kmd-dmd}.

We conclude this section with two remarks regarding the estimation of the eigenvalues of $\Koop$ and a connection to conditional mean embeddings~\cite{SHSF2009}.

\begin{remark}[Spectral Estimation]\label{rem:spectral}
Eq.~\eqref{eq:bound_efun} alone is not sufficient to derive strong guarantees on how well the spectra of $\Estim$  estimates the spectra of $\Koop$. While we address this in detail in App.~\ref{app:kmd-dmd}, here we comment that in general it may happen that $\TS_\im \Estim \approx\Koop\TS_\im$, while $\Spec(\Koop)$ is far from $\Spec(\Estim)$. If $\Koop$ is a normal compact operator, however, for $\Estim\in \HSr$ and  every $i\in[r]$, there exists $\lambda_{\im,i}\in\Spec(\Koop)$ such that $\abs{\lambda_i -\lambda_{\im,i}}\leq \norm{\TZ_\im-\TS_\im\Estim} \, \norm{\Estim} / \sigma_{r}(\TS_\im \Estim)$. If additionally $\Estim$ is also normal, then  $\abs{\lambda_i-\lambda_{\im,i}}\leq \sqrt{\ExRisk(\Estim)}\,\hnorm{\TS_\im}$. 
\end{remark}

\begin{remark}[Link to Conditional Mean Embeddings]\label{rem:kor_cme}
The Koopman operator is a specific form of conditional expectation operator and can be studied within the framework of conditional mean embeddings~\cite{SHSF2009}. 
Here,  
the goal is  to learn the function $\CME\colon\X\to\RKHS$ defined as 
\begin{equation}\label{eq:CME}
\CME(x):=\EE[\phi(X_{t+1})\,\vert\,X_t = x]=\int_{\X}p(x,dy)\phi(y), \quad x\in \X,
\end{equation}
called the \textit{conditional mean embedding} (CME) of the conditional probability $p$ into $\RKHS$. In App.~\ref{app:kor-cme} we show a ``duality'' between Koopman operator regression and CME expressed by the reproducing property $(\TZ_\im f)(x) = \scalarp{f,\CME(x)}$. In particular, recalling that $\rho$ is the joint probability measure on $\X \times \X$ defined by $\rho(dx,dy) = p(x, dy) \im(dx)$, the risk we proposed in \eqref{eq:true_risk} can be written as
\begin{equation}\label{eq:true_risk_cme}
\underbrace{ \EE_{(x,y)\sim \rho} \norm{\phi(y) - \Estim^*\phi(x)}^2}_{\Risk(\Estim)} =  \underbrace{\EE_{(x,y)\sim\rho}\norm{\CME(x) - \phi(y) }^2}_{\IrRisk} +  \underbrace{\EE_{x\sim \im} \norm{\CME(x) - \Estim^*\phi(x)}^2}_{\ExRisk(\Estim)}
\end{equation}
In this sense the Koopman operator regression problem \eqref{eq:KRP} is equivalent to learning CME of the Markov transition kernel $p$. 
\end{remark}

\section{Empirical Risk Minimization}\label{sec:erm}
We next describe different estimators for the Koopman operator. 
Let $\ES,\EZ\in\HS{\RKHS,\R^n}$ be the {\em sampling} operators of the inputs and outputs 
defined, for every $f \in \RKHS$, as 
$\ES f = \big(n^{-\frac{1}{2}} f(x_i)\big)_{i=1}^n$ 
and 
$\EZ f = \big(n^{-\frac{1}{2}}  f(y_i)\big)_{i=1}^n$, respectively. 
The Koopman operator is  estimated by minimizing, under different constraints, the {\em empirical risk}
\begin{equation}\label{eq:empirical_risk}
    \widehat{\Risk}(\Estim) :
    = \left \Vert \EZ - \ES\Estim \right \Vert_{{\rm HS}}^{2}
    = \frac{1}{n}\sum_{i\in[n]}\norm{\phi(y_{i}) - \Estim^*\phi(x_i)}^{2} ,\qquad \Estim \in \HS{\RKHS}.
\end{equation}
The first expression is the empirical version of the risk in~\eqref{eq:true_risk}, while the second expression
is the empirical version of the risk as in~\eqref{eq:true_risk_cme}, with a remark that sampling operator $\EZ$ is not an estimator of the regression operator $\TZ_\im$ and that, since $\im$ is an invariant measure, we have $\EE[\EZ^*\EZ]=\EE[\ES^*\ES]=\TS_\im^*\TS_\im$. Further, notice that for the linear kernel $\phi(x) = x$ Eq.~\eqref{eq:empirical_risk} is essentially 
the problem minimized by the classical DMD, whereas if $\phi$ is built from a dictionary of functions it is minimized by the {\em extended} DMD~\cite{Kutz2016}.

Before discussing further, we introduce the empirical input, output and cross covariances, given by 
$\ECx := \ES ^*\ES$,~ $\ECy :=  \EZ ^*\EZ$ and $\ECxy := \ES ^*\EZ$, respectively, and the corresponding kernel Gram matrices given by $\Kx  := \ES  \ES ^*$, $\Ky := \EZ  \EZ ^*$ and $\Kyx := \EZ  \ES ^*$. We also let $\ECx_\reg : = \ECx+\reg \Id_{\RKHS}$ be the regularized empirical covariance and $\Kx_{\reg} := \Kx + \reg \Id_n$ the regularized kernel Gram matrix. 

{\bf Kernel Ridge Regression (KRR).~} A natural approach is to add
a Tikhonov regularization term to~\eqref{eq:empirical_risk}
obtaining  
a {\em Kernel Ridge Rregression} (KRR) estimator
\begin{equation}\label{eq:HSnorm_KRR}
\EEstim_{\reg} := \argmin \big\{ \widehat{\Risk}(\Estim)+ \reg \left \Vert \Estim\right \Vert_{{\rm HS}}^{2} : \Estim \in \HS{\RKHS}\big\}.
\end{equation} 
It is easy to see that $\EEstim_\reg= \ECx_\reg^{-1} \ECxy = \ES^{*}\Kx_{\reg}^{-1}\EZ $. 
One issue with the above estimator is that 
the computation of its spectral decomposition becomes unstable with large datasets, see below. Consequently, low rank estimators have been  advocated~\cite{Kutz2016} as a way to overcome these limitations.

{\bf Principal Component Regression (PCR).~} A standard strategy to obtain a low-rank estimator is 
{\em  Principal Component Regression} (PCR). Here, 
 the input data are projected to the principal subspace of the covariance matrix $\ECx$,
 and  ordinary least squares on the projected data is performed, yielding 
the estimator $\EEstim_{r}^{{\rm PCR}}=  [\![ \ECx]\!]_r^{\dagger}\ECxy = \ES^{*}[\![\Kx]\!]_r^{\dagger}\EZ$.
 In the context of dynamical systems, this estimator is known as {\em kernel Dynamic Mode Decomposition}, and is of utter importance in a variety of applications~\cite{Kutz2016, Brunton2022}. Note, however, that PCR does  {\em not} minimize the empirical risk under the low-rank constraint.


{\bf Reduced Rank Regression (RRR).~} 
The optimal rank $r$ empirical risk minimizer is
\begin{equation}\label{eq:HSnorm_RRR_1}
\EEstim_{r, \reg} :=
\argmin \big\{ \widehat{\Risk}(\Estim)+ \reg \left \Vert \Estim\right \Vert_{{\rm HS}}^{2} : \Estim \in \HSr\big\}.
\end{equation} In classical linear regression this problem is known as  {\em reduced rank regression} (RRR) \cite{Izenman}. While  extensions to infinite dimensions have been considered \cite{mukherjee2011reduced,wang2020reduced}, we are not aware of any work considering the HS operator setting presented here. The minimizer of \eqref{eq:HSnorm_RRR_1} is given by $\EEstim_{r, \reg} = \ECx_\reg^{-\frac{1}{2}} [\![ \ECx_\reg^{-\frac{1}{2}}\ECxy ]\!]_r = \ES^* U_r V_r^\top \EZ$. Here $V_r = \Kx  U_r$ and $U_r = [u_1\, \vert \ldots \,\vert u_r]\in\R^{n\times r} $ is the matrix whose columns are 
the $r$ leading  eigenvectors of the generalized eigenvalue problem $\Ky \Kx  u_i = \sigma_i^2 K_\reg u_i$, normalized as $u_i^\top \Kx  K_\reg u_i = 1$, $i \in [r]$.

{\bf Forecasting and Modal Decomposition.~}
The above estimators are all of the form $\EEstim = \ES^* W \EZ$, for some $n\times n$ matrix $W$.
Given  $f \in \RKHS$, the one-step-ahead expected value  $\EE[ f(X_{t+1})\, \vert\, X_t=x]$ is estimated by 
\[[\EEstim f](x)= [\ES ^* W \EZ f](x)= \tfrac{1}{n}\sum_{i=1}^n (W f_n)_i k(x,x_{i}),
\]
where $f_n = (f(y_i))_{i=1}^n$ and we used the  definition of $\ES$ and $\EZ$.
Computing the above estimator is demanding in large scale settings, mostly because a large kernel matrix needs be stored and manipulated. 
Predicting with KRR, therefore, is tantamount to solve a linear system of dimension $n$, whereas for rank $r$ estimators we only need matrix multiplications. Notice that the computational complexity of both PCR and RRR estimators is of order $\bigO(r^2 n^2)$, which is better than the $\bigO(n^3)$ complexity of KRR. 
However, a number of recent ideas for scaling kernel methods can be applied for KRR, see e.g. \cite{billions,tropp} and references therein. Perhaps more importantly, specific to the context of dynamical systems is the fact that an approximate mode decomposition needs to be further computed, requiring the spectral decomposition of $\EEstim$. As showed in App.~\ref{app:comp_dmd}, this reduces to computing the spectral decomposition of an $n\times n$ matrix $W\Kyx$, where recall $\Kyx = \EZ  \ES ^*$. 
For KRR computing the spectral decomposition, in general, has complexity $\bigO(n^3)$ and may become numerically ill-conditioned for theoretically optimal regularization parameters. 
In contrast, the low rank structure of PCR and RRR estimators allows efficient and numerically stable spectral computation of complexity $\bigO(r^2 n^2)$ as we show next. 
\begin{restatable}{theorem}{thmDMD}\label{thm:DMD_theorem}
Let $\EEstim = \ES^* U_r V_r^\top \EZ$, with $U_r, V_r \in \R^{n\times r}$. If~$V_r^\top \Kyx U_r \in\R^{r\times r}$ is full rank and non-defective, the spectral decomposition $(\lambda_i,\lefun_i, \refun_i)_{i\in[r]}$ of $\EEstim$ can be expressed in terms of the spectral decomposition $(\lambda_i,\levec_i, \revec_i)_{i\in[r]}$~of~ $V_r^\top \Kyx U_r $.  Namely,~$\lefun_i = \EZ^*V_r \levec_i / \overline{\lambda}_i$ and $\refun_i = \ES^*U_r \revec_i$, for all $i\in[r]$. In addition, for every $f\in\RKHS$, its dynamic modes are given by $\gamma_i^f = \levec_i^* V_r^\top f_n / (\lambda_i\sqrt{n})\in\C$.

\end{restatable}

In addition to the estimators presented above, several other popular DMD methods are captured by our Koopman operator regression framework. In App. \ref{app:erm} we review some of them providing also the proof of Thm.~\ref{thm:DMD_theorem}. We now turn to the study of risk bounds for the proposed low rank estimators. 

\section{Learning Bounds}
\label{sec:bounds} 
In this section, we bound the deviation of the risk from the empirical risk, uniformly over a prescribed set of HS operators on $\RKHS$. The analysis here is presented for Ivanov regularization for simplicity, but our results can be linked to Tikhonov regularization (see \cite{luise2019} and reference therein for a discussion). To state the result, we denote (true) input, output and cross covariances by $\Cx := \EE_{x\sim\im}[\phi(x)\otimes\phi(x)]$, $\Cy :=  \EE_{y\sim\im}[\phi(y)\otimes\phi(y)]$ and $\Cxy := \EE_{(x,y)\sim\rho}[\phi(x)\otimes\phi(y)]$, respectively. Note that since $\im$ is invariant measure, the input covariance and output covariance are the same $\Cx=\Cy = \TS_\im^*\TS_\im$, while for the cross covaraince we have $\Cxy = \TS_\im^*\TZ_\im$.  Without loss of generality we present the results in the case that $\|\phi(x)\| \leq 1$, for all $x \in \X$ (the bounds below need otherwise to be rescaled by a constant).

We start by presenting a theorem holding in the setting where the data is sampled i.i.d. from the joint probability measure $\rho(dx, dy) = \im(dx)\transitionkernel(x,dy)$.
\begin{restatable}{theorem}{thmmain}
\label{thm:UB_main_text}
Let $\G_{r,\reg} = \{\Estim \in \HSr ~:~ \|\Estim\|_{\rm HS} \leq \reg\}$
and define 
$\sigma^2 = \EE (\|\phi(y)\|^2 -\EE \|\phi(y)\|^2)^2$. With probability at least $1-\delta$ in the i.i.d. draw of $(x_i,y_i)_{i=1}^n$ from $\rho$,
we have for every $\Estim \in \G_{\reg,r}$
\[
|\Risk(\Estim) {-} \widehat{\Risk}(\Estim)| \leq  
\sqrt{\frac{2\sigma^2 \ln \frac{6}{\delta}}{n}} + 3(4 \sqrt{2r}\gamma \hspace{.05truecm} {+} \hspace{.05truecm}\gamma^2) \sqrt{\frac{\|\Cx\| \ln \frac{24n^2}{\delta}}{n}}  + \frac{(1 \hspace{.05truecm}{+} \hspace{.05truecm}24\gamma \sqrt{r}) \ln \frac{6}{\delta} \hspace{.05truecm} {+} \hspace{.05truecm} 6\gamma^2  \ln \frac{24n^2}{\delta}}{n}. 
\]
\end{restatable}

A key tool in the proof is the following proposition, which is a natural extension of \cite[Theorem 7]{MauPon13} who provided concentration inequalities for classes of positive operators.
\begin{restatable}{proposition}{propone}
\label{prop:11_main}
With probability at least $1-\delta$ in the i.i.d. draw of $(x_i,y_i)_{i=1}^n$ from $\rho$,
\[
\left\|\ECxy - \Cxy \right\|
\leq 
12 \frac{\ln \frac{8n^2}{\delta}}{n} + 6 \sqrt{\frac{ \|C\| \ln \frac{8n^2}{\delta} }{n}}.
\]
\end{restatable}
\begin{proof}[{\it Sketch of the proof of Thm. \ref{thm:UB_main_text} }]~Recalling the definition of the true and empirical risk, a direct computation gives
\begin{equation*}
\Risk(\Estim) - \widehat{\Risk}(\Estim) =  \tr\big[ \Cy - \ECy \big] + \tr  \big[\Estim\Estim^* (\Cx - \ECx)\big] - 2  \tr \big[
\Estim^*(
\Cxy- \ECxy) \big].
\end{equation*}
We use H\"older inequality to bound the last two terms in the r.h.s., obtaining
\begin{equation}
\label{eq:ggg-mp}
 \Risk(\Estim) - \widehat{\Risk}(\Estim) \leq  \tr\big[ \Cy - \ECy \big]+ \reg^2 \|\Cx - \ECx\|  + 2  \sqrt{r} \reg \|\Cxy -\ECxy\|.
\end{equation}
We then bound the first term as $\tr[\Cy-\ECy]\leq \frac{\ln \frac{2}{\delta}}{3n} + \sqrt{\frac{2 \sigma^2 \ln \frac{2}{\delta}}{n}}$, see App. \ref{app:bounds_unif_iid}, use \cite[Theorem~7-(i)]{MauPon13} to bound the second term, and Prop.~\ref{prop:11_main} for the last term. The result then follows by a union bound.
\end{proof}
We state several remarks on this theorem and its implications.
\begin{enumerate}[leftmargin=.5truecm]
\itemsep0em
\item 
It is interesting to compare the bound for RRR and PCR estimators. Assuming that both estimators have the same HS norm, then they will satisfy the same uniform bound. However, the empirical risk may be (possibly much) smaller for the RRR estimator and hence preferable, see Fig.~\ref{fig:uniform_bound_logistic}. 
\item  
Using the reasoning in \cite{maurer2009,MauPon13} and \cite{mnih2008} we can replace the variance term and the term $\|\Cx\|$ in the bound with their empirical estimates, obtaining a fully data dependent bound. Notice also that the bound readily applies to the more general CME case, which could be subject of future work.

\item 
A related bound can be derived using Cor.~3.1 in \cite{minsker2017} in place of Prop.~\ref{prop:11_main}. This essentially replaces the term $\|C\|$ with $\| \EE AA^*\|$ where $A:=(\phi(x)\otimes \phi(y) - T)$.
This bound is more difficult to turn into a data dependent bound, but it allows for a more direct comparison to bounds without the rank constraint, which may be potentially much larger; see the discussion in  App.~\ref{app:bounds_unif_iid}.
\item  One can use the uniform bound to obtain an excess risk bound. In the setting of Thm.~\ref{thm:UB_main_text} and well specified case in which $\TZ_{\im} = \TS_\im\HKoop $ this requires studying the approximation error 
$\min_{\Estim \in \G_{\gamma,r}} \|\TS_\im (\HKoop- \Estim)\|_{\rm HS}^2$. 
\end{enumerate}

\paragraph{Dealing with Sampled Trajectories.~}We now study ERM with time dependent data. We consider that a trajectory $x_1,\dots,x_{n+1}$ has been sampled from the process as 
$x_1\sim \im,y_{k-1}=x_k \sim \transitionkernel(x_{k-1},\cdot)$, $k \in [2{:}n]$.  
For a strictly stationary Markov process the $\beta $-mixing coefficients
are the numbers $\beta _{\mathbf{X}}\left( \tau \right) $ defined for $\tau
\in \mathbb{N}$ by
\[
\beta _{\mathbf{X}}\left( \tau \right) =\sup_{B\in \Sigma \otimes \Sigma}\left\vert \rho_\tau\left( B\right) -(\pi \times \pi)\left( B\right) \right\vert ,
\]
where $\rho_\tau$
 is the joint distribution of $X_{1}
$ and $X_{1+\tau }$. The basic strategy, going back to at least \cite{yu1994rates}, to transfer a concentration result for i.i.d. variables to the non-i.i.d.
case represents the process $\mathbf{X}$ by interlaced block-processes $%
\mathbf{Y}$ and $\mathbf{Y}^{\prime }$, which are constructed in a way that $%
Y_{j}$ and $Y_{j+1}$ are sufficiently separated to be regarded as
independent. Specifically, they are defined as
\[
Y_{j}=\sum_{i=2\left( j-1\right) \tau +1}^{\left( 2j-1\right) \tau }X_{i}%
\quad\text{ and }\quad Y_{j}^{\prime }=\sum_{i=\left( 2j-1\right) \tau +1}^{2j\tau
}X_{i}\quad\text{ for }j\in \mathbb{N}.
\]%
This construction naturally yields the following key lemma, which allows us to extend 
several results from the i.i.d. case to time dependent stationary Markov chains. 
The proof is presented in App.~\ref{app:bounds_unif_mix}.

\begin{restatable}{lemma}{lemkey}\label{lem:blockprocess}
Let $\mathbf{X}$ be strictly stationary with
values in a normed space $\left( \mathcal{X},\left\Vert \cdot\right\Vert \right) 
$, and assume $n=2m\tau$ for $\tau ,m\in \mathbb{N}$. Moreover, let $Z_{1},\dots,Z_{m}$ be $m$ independent copies of $Z_{1}=\sum_{i=1}^{%
\tau }X_{i}$.
Then for $s>0$
\[
\PP \Big\{ \Big\| \sum_{i=1}^n X_{i}\Big\| >s\Big\} \leq 2\,\PP
\Big\{ \Big\| \sum_{j=1}^m Z_{j}\Big\| >\frac{s}{2}\Big\}
+2\left( m-1\right) \beta _{\mathbf{X}}\left( \tau \right).
\]%
\end{restatable}
As an application of this result we transfer Prop. \ref{prop:11_main}, which was key in the proof of Thm.~\ref{thm:UB_main_text} to give an estimation bound for $\|\Cxy -\ECxy\|$, to the non-i.i.d. setting. Fix $\tau \in \mathbb{N}$ and let $Z_{1},\dots,Z_{m}$
be independent copies of $Z_{1}=\frac{1}{\tau }\sum_{i=1}^{\tau }\phi(x_{i})\otimes \phi(x_{i+1}) - T$.
Applying Lem.~\ref{lem:blockprocess} with $\phi(x_i)\otimes\phi(x_{i+1}) - T$ in place of $X_i$ we obtain the following.
\begin{restatable}{proposition}{proponebis}
\label{prop:11_main_bis}
Let $\delta > (m-1) \beta _{\mathbf{X}}( \tau -1)$. With probability at least $1-\delta$ in the draw $x_1\sim \im,x_i \sim \transitionkernel(x_{i-1},\cdot)$, $i \in [2{:}n]$,
\[
\norm{\ECxy - \Cxy}\leq \frac{48}{m}\ln \frac{4m\tau }{\delta -\left( m-1\right) \beta _{\mathbf{X}}\left( \tau -1\right) } + 12 \sqrt{\frac{2\left\Vert C \right\Vert}{m}\ln \frac{4m\tau }{\delta -\left( m-1\right) \beta _{\mathbf{X}}\left( \tau -1\right) } } .
\]
\end{restatable}
We notice that, apart from slightly larger numerical constants and a logarithmic term, Prop.~\ref{prop:11_main_bis} is conceptually identical to Prop.~\ref{prop:11_main} provided the sample size $n$ is replaced by the effective sample size $m\approx n/2\tau$. Similar conclusions can be made to bound the other random terms appearing in \eqref{eq:ggg-mp}, see
App.~\ref{app:bounds_unif_mix} for a discussion.

\section{Experiments}\label{sec:exp}

In this section we show that the proposed framework  
can be applied to dissect and forecast dynamical systems. While we keep the presentation concise, all the technical aspects, as well as additional experiments, are deferred to App.~\ref{app:exp}. Along with the code to reproduce the experiments, at the url \href{https://github.com/CSML-IIT-UCL/kooplearn}{https://github.com/CSML-IIT-UCL/kooplearn}, we release a Python module implementing \texttt{sklearn}-compliant~\cite{scikit-learn} estimators to learn the Koopman operator.

{\bf 
Noisy Logistic Map.~} 
We study the noisy logistic map, a non-linear dynamical system defined by the recursive relation $x_{t + 1} = (4x_{t}(1 - x_{t}) + \xi_{t}) \mod 1$ over the state space $\X = [0 , 1]$. Here, $\xi_{t}$ is i.i.d. additive {\em trigonometric} noise as defined in~\cite{Ostruszka2000}. The probability distribution of trigonometric noise is supported in $[-0.5,0.5]$ and is proportional to $\cos^{N}(\pi\xi)$, $N$ being an {\em even} integer. In this setting, the true invariant distribution, transition kernel and Koopman eigenvalues are easily computed. In Tab.~\ref{tab:logistic_map_results} we compare the performance of KRR, PCR and RRR (see Sec.~\ref{sec:erm}) trained with a Gaussian kernel. We average over 100 different training datasets each containing $10^{4}$ data points and evaluate the test error on 500 unseen points. In Tab.~\ref{tab:logistic_map_results} we show the approximation error for the three largest eigenvalues of the Koopman operator, $\lambda_{1} = 1$ and $\lambda_{2,3} = -0.193 \pm 0.191i$ as well as training and test errors. The following eigenvalues $|\lambda_{4,5}| \approx 0.027$ are an order of magnitude smaller than $|\lambda_{2,3}|$.  Both PCR and RRR have been trained with the rank constraint $r=3$. The regularization parameter $\reg$ for KRR and RRR is the value $\reg \in [10^{-7}, 1]$ minimizing the validation error. The RRR estimator always outperforms PCR, and in the estimation of the non-trivial eigenvalues $\lambda_{2,3}$ ($\lambda_{1}$ corresponding to the equilibrium mode is well approximated by every estimator) attains the best results. In Fig.~\ref{fig:uniform_bound_logistic} we report the results of a comparison between PCR and RRR performed under Ivanov regularization. This experiment was designed to empirically test the uniform bounds presented in Sec.~\ref{sec:bounds}. Again, RRR consistently outperforms the PCR estimator.

{\small
\begin{figure}[t!]
\begin{center}
\includegraphics[width=0.7\textwidth]{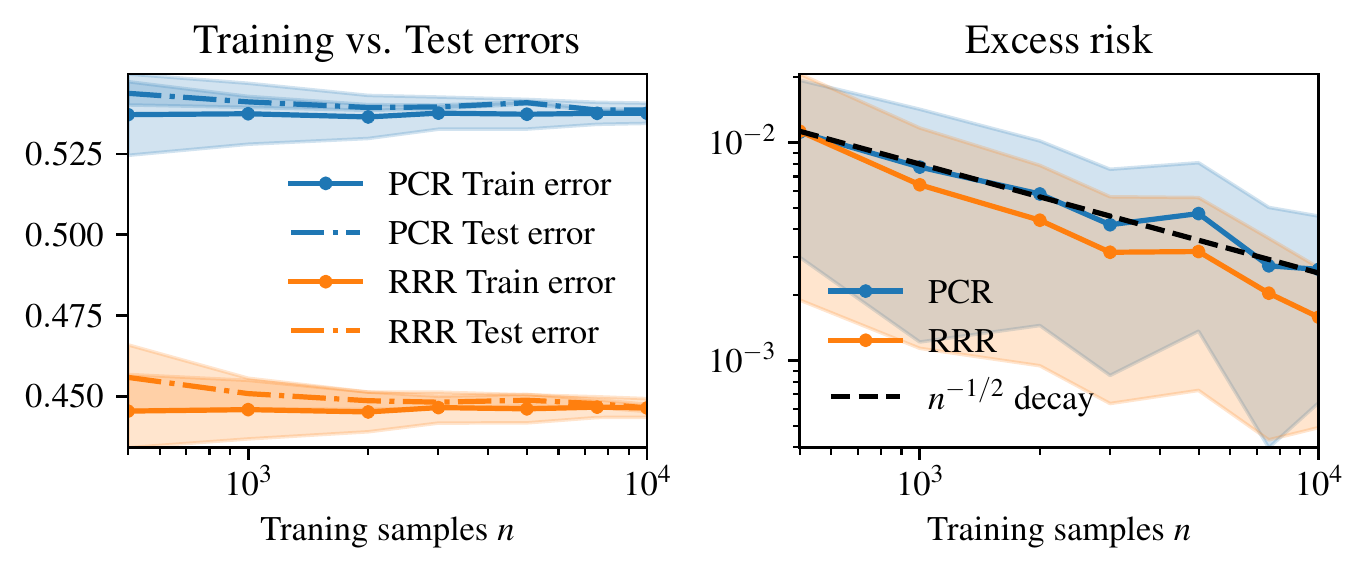}
\end{center}
\caption{Numerical verification of the uniform bound presented in Thm.~\ref{thm:UB_main_text} for the noisy Logistic map. Left panel: the training and test risk for RRR are consistently than PCR. Right panel: the deviation between training and test risk 
as a function of the number of training samples have $\approx n^{-1/2}$ decay.}  \label{fig:uniform_bound_logistic}
\end{figure}
}

\begin{table}[h]
    \small
    \caption{Comparison of the estimators proposed in Section~\ref{sec:erm} on the noisy logistic map.
    }\label{tab:logistic_map_results}
    \centering
    \begin{tabular}{r|cccc}
        \toprule
        Estimator & Training error & Test error & $|\lambda_{1} - \hat{\lambda}_{1}|/|\lambda_{1}|$ & $|\lambda_{2,3} - \hat{\lambda}_{2,3}|/|\lambda_{2,3}|$ \\ 
        \midrule
        PCR & $  0.2 \pm 0.003$ & $ 0.18 \pm 0.00051$ & $9.6 \cdot 10^{-5} \pm 7.2 \cdot 10^{-5}$ & $ 0.85 \pm 0.03$ \\
        RRR & $ 0.13 \pm 0.002$ & $\bm{0.13 \pm 0.00032}$ & $5.1 \cdot 10^{-6} \pm 3.8 \cdot 10^{-6}$ &$\bm{0.16 \pm 0.1}$ \\
        KRR & $\bm{0.032 \pm 0.00057}$ & $ \bm{0.13 \pm 0.00068}$ & $\bm{7.9 \cdot 10^{-7} \pm 5.7 \cdot 10^{-7}}$ & $ 0.48 \pm 0.17$ \\
        \bottomrule
    \end{tabular}
\end{table}


{\bf 
The molecule Alanine Dipeptide.~} 
We analyse a simulation of the small molecule Alanine dipeptide reported in Ref.~\cite{WN2018}.
The dataset, here, is a time series of the Alanine dipeptide atomic positions. The trajectory spans an interval of 250 ns and the number of features for each data point is 45. The dynamics is Markovian and governed by the Langevin equation~\cite{Lemons1997}. The system supports an invariant distribution, known as the Boltzmann distribution, and the equations are time-reversal-invariant. The latter implies that the true Koopman operator is self-adjoint and has real eigenvalues. For Alanine dipeptide it is well known that the dihedral angles play a special role, and characterize the {\em state} of the molecule. 
Broadly speaking, we can associate specific regions of the dihedral angles space to metastable states, i.e. configurations of the molecule which are "stable" over an appreciable span of time. To substantiate this claim we point to the left panel of Figure~\ref{fig:CNN_kernel}. From this plot it is evident that the molecule spend a large amount of time around specific values of the angle $\psi$, and transitions from one region to another are quite rare. We now try to recover the same informations from the spectral decomposition of the Koopman operator. We train the RRR estimator with rank 5 and a standard Gaussian kernel (length scale $\ell = 0.2$). We remark that the dataset is comprised of atomic positions, and not dihedral angles. We show that the computed eigenfunctions are highly correlated with the dihedral angles, meaning that our estimator was able to {\em learn the correct physical quantities} starting only from the raw atomic positions. The estimated eigenvalues are $\lambda_{1} = 0.99920$, $\lambda_{2} = 0.9177$, $\lambda_{3} = 0.4731$, $\lambda_{4} = -0.0042$ and $\lambda_{5} = -0.0252$. Notice that they are all real (as they should be, since the system is time-reversal-invariant). In Fig.~\ref{fig:ala2_lead_eigenfunctions} of App.~\ref{app:ala2} we report the plots of the (non-trivial) eigenfunctions corresponding to $\lambda_{2}$ and $\lambda_{3}$ in the dihedral angle space. From these plots it is clear that the eigenfunctions are to a good approximation piecewise constant and identify different metastability regions, as expected.

{\bf Koopman Operator Regression with Deep Learning Embeddings.~} To highlight the importance of choosing the kernel, here we consider a computer vision setting, where standard  kernels (e.g. Mat\'ern or Gaussian),  are less suitable than features given by pre-trained deep learning models. We take a sequence $(x_t)_{t\in\N}$ of images from the MNIST dataset \cite{DengMNIST2012} starting from $x_0$ corresponding to an image depicting a digit $0$ and such that for every $x_t$ depicting a digit $c_t\in\{0,\dots,9\}$ we sample $x_{t+1}$ from the set of images depicting the digit $c_t+1~ (\textrm{mod}~ 10)$. We compare the rank-10 RRR estimators using Linear and Gaussian kernels, with a {\itshape Convolutional Neural Network (CNN) kernel} $k_{{\bm \theta}}(x, x^{\prime}) := \left\langle \phi_{{\bm \theta}}(x), \phi_{{\bm \theta}}(x^{\prime})\right\rangle$, where $\phi_{{\bm \theta}}$ is a feature map obtained from the last layer of a convolutional neural network classifier trained on the same images in $(x_t)_{t\in\N}$ 
. We trained the three Koopman estimators on $1000$ samples. The right panel of Fig.~\ref{fig:CNN_kernel} shows the first $4$ forecasting steps starting from a digit $0$. Only the forecasts by the CNN kernel maintain a sharp (and correct) shape for the predicted digits. In contrast, the other two kernels are less suited to capture visual structures and their predictions quickly lose any resemblance to a digit. This effect can be appreciated starting from other digits, too (see App.~\ref{app:exp}). 
{\small
\centering
\begin{figure}[t!]
\hspace{1.5truecm}\includegraphics[width=10truecm]{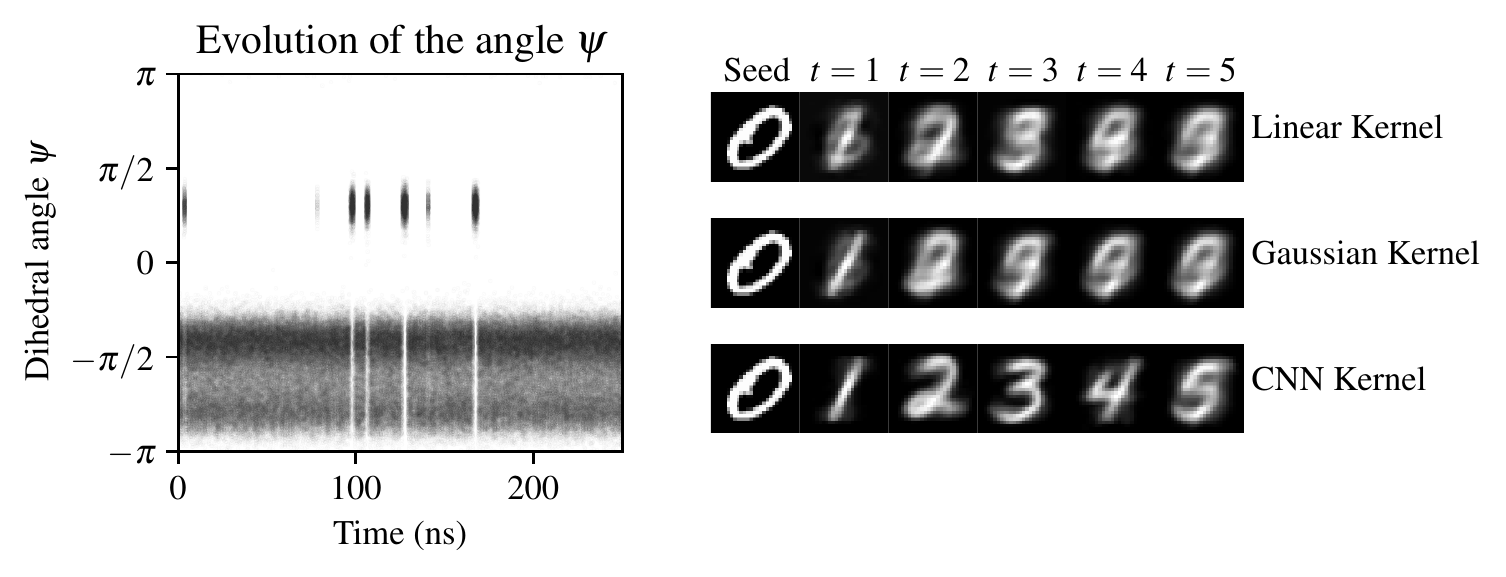}
    \caption{Left: dihedral angle $\psi$ of the Alanine dipeptide as a function of time. Right: comparison of different kernels in the generation of a series of digits. Starting from a seed image, the next ones are obtained by iteratively using the rank-10 RRR Koopman operator estimator.}\label{fig:CNN_kernel}
\end{figure}
}

\section{Conclusions}
We proposed a statistical framework to learn Koopman operators in RKHS. In addition, we 
investigated how the spectral and modal decompositions of the Koopman estimators approximate the true ones, providing novel perturbation bounds. In particular, we studied three Koopman operator estimators, KRR, PCR and the newly proposed RRR. We observed that KRR and PCR correspond to well-established estimators from the dynamical systems literature. Then, by leveraging recent results from kernel operator learning we observed that such estimators enjoy strong statistical guarantees. Focusing on the RRR estimator, we provided generalization bounds, both in i.i.d. and non i.i.d. settings, one of the key novel contributions of this work.

In this work we consider only time-homogeneous dynamical systems admitting an invariant distribution. Weakening these assumptions would allow the study of more general systems. Moreover, the extension of our results to continuous dynamical systems and in general to non-uniformly sampled datapoints deserves further investigations. Finally, the choice of the kernel is fundamental for efficient learning. Designing a kernel function incorporating prior knowledge of the dynamical system (such as structure of the data points, symmetries, smoothness assumptions etc.) is a topic of paramount interest.


\paragraph{Acknowledgements:}
This work was supported in part by the EU Projects ELISE and ELSA. We thank all anonymous reviewers for their useful insights and suggestions.

\bibliographystyle{apalike}
{
\bibliography{bibliography}
}

\section*{Checklist}


\newpage
    
\appendix

\begin{center}
{\Large \bf Supplementary Material} 
\end{center}

Below we give an overview of the structure of the supplementary material.

\begin{itemize}
    \item  App.~\ref{app:background} contains additional technical background on Markov processes and Koopman operators, notably on Koopman Mode Decomposition.   
\item In App.~\ref{app:learning} we provide detailed proofs of the results presented in Sec. \ref{sec:learning}. In particular, we provide bounds on the distance between spectra of the Koopman operator and its estimation in App. \ref{app:kmd-dmd}, and discuss a duality between Koopman operator regression (KOR) and conditional mean embeddings (CME) in App. \ref{app:kor-cme}.  

\item In App.~\ref{app:erm} we expand the content of Sec. \ref{sec:erm}. In App. \ref{app:comp_estim} we discuss the computation of three estimators considered in this work, while in App. \ref{app:comp_dmd} we show how to 
compute their modal decompositions.

\item In App.~\ref{app:bounds} we prove the statistical learning bounds presented in Sec.~\ref{sec:bounds} and briefly discuss their implications and future research directions. 

\item Finally, in App.~\ref{app:exp} we provide more 
details on the experimental section, as well as present  additional experiments.

\end{itemize}

\section{Background on the Koopman Operator Theory}\label{app:background}
We now recall basic results concerning the theory of Koopman (i.e. transfer) operators. As mentioned in the main text, the natural function space $\F$ in which the Koopman operator can be defined is $\F = L^{\infty}(\X)$. In this case, given a transition kernel $\transitionkernel$,  by integrating, we can define the {\em transfer operator} acting either on $L^{\infty}(\X)$-measurable functions (from the right) or $\sigma$-finite measures on $\sigalg$ (from the left).
\begin{definition}[Transfer operator]
    We define the linear transfer operator $\transferop$ acting on the right on functions $f \in L^{\infty}(\X)$
    \begin{subequations}
        \begin{equation}\label{eq:transfer_operator_right}
            (\transferop f)(x) := \int_{\X} p(x, dy)f(y) = \mathbb{E}\left[f(X_{i + 1}) \middle | X_{i} = x\right]
        \end{equation}
        and on the left on $\sigma$-finite measures on $\sigalg$
        \begin{equation}\label{eq:transfer_operator_left}
            (\mu \transferop)(B) := \int_{\X} \mu(dx)p(x,B) \qquad B \in \sigalg.
        \end{equation}
\end{subequations}
\end{definition}
We notice that~\eqref{eq:transfer_operator_right} acts exactly as the Koopman operator defined in the main text, although on a different function space. Equation~\eqref{eq:transfer_operator_left}, on the other hand, can be interpreted as evolving distributions. Indeed, given an initial distribution of states $\mu$, evolving each state for one step forward, will yield the distribution $\mu\transferop$. If the transition kernel is {\em non-singular}, that is for all $B \in \sigalg$ such that $\mu(B) = 0$ one has $(\mu \transferop)(B) = 0$, in view of the Radon–Nikodym theorem we also have that~\eqref{eq:transfer_operator_left} can be interpreted as the adjoint of~\eqref{eq:transfer_operator_right} with respect to the Banach duality pairing between $L^{\infty}(\X, \mu)$ and $L^{1}(\X, \mu)$, see e.g.~\cite{Lasota1994}. From~\eqref{eq:transfer_operator_left} it also follows that $\im$ being an invariant distribution means that $\im \transferop = \im$, i.e. it is a fixed point of the transfer operator acting on the left. The following lemma proves that if $\im$ is an invariant distribution,~\eqref{eq:transfer_operator_right} is a non-expansive operator in every Lebesgue space $L^{q}(\X, \im)$ with $1 \leq q < \infty$.

\begin{lemma}\label{lemma:transfer_operator_L2}
    If $\im$ is an invariant probability measure, the operator $\transferop$ is a weak contraction on $L_{q}(\X, \im)$ for all $1 \leq q < \infty$. Additionally, it holds that  $\norm{\transferop} = 1$. 
\end{lemma}
\begin{proof}
For the first part, Jensen's inequality and the invariance of $\im$ directly give
    \begin{equation*}
        \begin{split}
            \int_{\X} |(\transferop f)(dx)|^{q} \im(dx) & =  \int_{\X} \im(dx) \left(\int_{\X}\transitionkernel(x, dy) |f(y)| \right)^{q} 
            \leq  \int_{\X} \im(dx) \int_{\X}\transitionkernel(x, dy) |f(y)|^{q}
            \\ & = \int_{\X} (\im\transferop)(dy) |f(y)|^{q}
            = \int_{\X} \transferop(dy) |f(y)|^{q}.
        \end{split}
    \end{equation*}
For the second part we notice that for any constant function $c$, one has  $\norm{c} = \norm{\transferop c} \leq \norm{\transferop}\norm{c}$, that is $\norm{\transferop} \geq 1$. This fact coupled with the first part of the lemma yields $\norm{\transferop} = 1$.
\end{proof}
\begin{corollary}
If $\im$ is an invariant probability measure, the operator $\transferop$ is well defined in $L_{q}(\X, \im)$ for all $1 \leq q < \infty$, and in particular for $q=2$, the case explored in the main text.
\end{corollary}

We just proved that whenever an invariant probability measure $\im$ exists,~\eqref{eq:transfer_operator_right} can be defined directly in $\Lii$. An interesting question is therefore what is the equivalent of~\eqref{eq:transfer_operator_left}, seen as the adjoint  of~\eqref{eq:transfer_operator_right} with respect to the Banach duality pairing. To characterize the adjoint operator $\transferop^*$,  we define the {\em time reversal} of $\transitionkernel$ as the Markov transition kernel $\transitionkernel^*(x,B):= \mathbb{P}\left\{X_{t - 1} \in B \middle| X_{t} = x \right\}$, and a simple calculation shows that $\transferop^*\colon\Lii\to\Lii$ is given by:
 \begin{equation}\label{eq:transfer_operator_adjoint}
 (\transferop^* f)(x) := \int_{\X} p^*(x, dy)f(y),
 \end{equation}
which can be seen as the {\em backward} transfer operator $[\transferop^*f](x) = \EE[ f(X_{t-1})\,\vert\,X_t = x]$. Notice that when the transfer operator on $\Lii$ is self-adjoint, i.e. $P=P^*$, the Markov chain is called \textit{time-reversal invariant} which is a relevant case in various fields such as physics and chemistry~\cite{Schutte2001}.

The following example shows that the basic tools developed in the classical theory of (deterministic) dynamical systems~\cite{Lasota1994} can be easily recovered in terms of transfer operators. \begin{example}[Deterministic Dynamical System]\label{ex:deterministic}
    Let $X_{i + 1} = \dynmap(X_{i})$ for all $i$, with $\dynmap:\X \to \X$. Clearly, the transition kernel for this Markov chain is
    \begin{equation}
        \transitionkernel(x, B) = \begin{cases}
            1 \qquad \text{ if } \dynmap(x) \in B \\
            0 \qquad \text{ otherwise }
        \end{cases}.
    \end{equation}
    This corresponds to $\transitionkernel(\cdot, A) = \mathds{1}_{B}\circ \dynmap = \mathds{1}_{\dynmap^{-1}(B)}$, which in turn implies that 
    \begin{equation}
        (\mu\transferop)(A) = \int_{\dynmap^{-1}(A)} \mu(dx).
    \end{equation}
    This is the Perron-Frobenius operator \cite{Lasota1994} as defined in the classical theory of dynamical systems. Analogously, $p(x, \cdot) = \delta_{\dynmap(x)}$ (the Dirac measure centered at $\dynmap(x)$) and 
    \begin{equation}
        (\transferop f)(x) = f(\dynmap(x))
    \end{equation}
    is the deterministic Koopman operator \cite{Lasota1994}. When an invariant measure $\im$ exists, the Koopman operator defined in $\Lii$ is known to be unitary~\cite{Budisic2012} and hence normal. In this respect, see~\cite{Gonzalez2021}, where general misconceptions on the Koopman operator (such as the one of always being a unitary) are discussed in detail.  
\end{example}

We conclude this section recalling the notion of spectra of linear operators. Let $T$ be a bounded linear operator on some Hilbert space $\mathcal{H}$. The {\em resolvent set} of the operator $T$ is defined as 
\begin{equation*}
    {\rm Res}(T) := \left\{\lambda \in \C \colon T - \lambda \Id \text{ is bijective} \right\}.
\end{equation*}
If $\lambda \notin {\rm Res}(T)$, then $\lambda$ is said to be in the {\it spectrum} $\Spec(T)$ of $T$. Recalling that $T - \lambda\Id$ bijective implies that it has a {\it bounded} inverse $(T - \lambda\Id)^{-1}$,
in infinite-dimensional spaces we can distinguish three subsets of the spectrum:
\begin{enumerate}
    \item Any $x \in \mathcal{H}$ such that $x \neq 0$ and $Tx = \lambda x$ for some $\lambda \in \C$ is called an {\em eigenvector} of $T$ with corresponding {\em eigenvalue} $\lambda$. If $\lambda$ is an eigenvalue, the operator $T - \lambda\Id$ is not injective and $\lambda \in \Spec(T)$. The set of all eigenvalues is called the {\em point spectrum} of $T$.
    \item The set of all $\lambda \in \Spec(T)$ for which $T - \lambda \Id$ is not surjective and the range of $T - \lambda\Id$ is dense in $\mathcal{H}$ is called the {\em continuous spectrum}.
    \item The set of all $\lambda \in \Spec(T)$ for which $T - \lambda \Id$ is not surjective and the range of $T - \lambda\Id$ is not dense in $\mathcal{H}$ is called the {\em residual spectrum}. 
\end{enumerate}
Finally if $T$ is a {\em compact} operator, the Riesz-Schauder theorem~\cite{Reed1980}, assures that $\Spec(T)$ is a discrete set having no limit points except possibly $\lambda = 0$. Moreover, for any nonzero $\lambda \in \Spec(T)$, then $\lambda$ is an {\em eigenvalue} (i.e. it belongs to the point spectrum) of finite multiplicity.  

\section{Learning Theory in RKHS}\label{app:learning}
%

We begin by proving the identity  \eqref{eq:true_risk}, which we restate in the following proposition.
\begin{proposition}\label{prop:risk}
Let $\Estim\in\HS{\RKHS}$ and let $(h_i)_{i\in\N}$ be complete a orthonormal system of $\RKHS$, then
\begin{equation}\label{eq:prop_risk_1}
 \sum_{i\in\N}\EE_{(x,y)\sim\rho}\left[(\TS_{\im}h_i)(y) - (\TS_{\im}\Estim h_i)(x)\right]^{2} = \hnorm{\TS_\im}^2 - \hnorm{\TZ_\im}^2 +  \hnorm{\TZ_\im - \TS_\im \Estim
 }^2.
\end{equation}
\end{proposition}

\begin{proof}
Given $A\colon\RKHS\to\Lii$ for an arbitrary $h\in\RKHS$, denoting $f=\TS_\im h$, we have that
\begin{equation*}
\EE_{(x,y)\sim\rho}\left[f(y) - (Ah)(x)\right]^{2} =\int_{\X \times \X} \im (dx)\transitionkernel(x, dy) \Big( (f(y)^2 - 2f(y)(Ah)(x) + (Ah)(x)^2 \Big). 
\end{equation*}
Using that $\pi(dy) = \int_{\X} \pi(dx)p(x,dy)$, we have both 
\[
\int_{\X \times \X} \im (dx)\transitionkernel(x, dy) f(y) = \int_{\X} \im (dy)f(y)\,\text{ and }\, \int_{\X \times \X} \im (dx)\transitionkernel(x, dy) f(y)^2 = \int_{\X} \im (dy)f(y)^2.
\]
A direct computation then gives that
\begin{equation*}
\EE_{(x,y)\sim\rho}\left[f(y) - (Ah)(x)\right]^{2} =
\norm{f}^2-\norm{\Koop f}^2 + \norm{\Koop f - Ah}^2.
\end{equation*}
Replacing $A$ and in the above expression by $\TS_\im \Estim$ and summing over $i\in\N$, we obtain
\begin{equation*} 
\sum_{i\in\N}\EE_{(x,y)\sim\rho}\left[(\TS_{\im}h_i)(y) - (\TS_{\im}\Estim h_i)(x)\right]^{2} = \hnorm{\TS_\im}^2 - \hnorm{\TZ_\im}^2 +  \hnorm{\TZ_\im - 
\TS_\im \Estim
}^2.
\end{equation*}
Now, 
since $\hnorm{\TZ_\im-\TS_\im \Estim}<\infty$, by Tonelli's theorem we can exchange summation and expectation $\EE_{x\sim\im}$, and the proof is completed. We remark that in the risk definition~\eqref{eq:true_risk} in the main text, we slightly abused notation as $h_{i} \in \RKHS$, but the expectation value is defined in $\Lii$. The formally correct version of~\eqref{eq:true_risk} is obtained with the substitution $h_{i} \mapsto \TS_{\im}h_{i}$.
\end{proof}

Next, we prove our main result on the approximation of the Koopman operator via RKHS.  We show that if an RKHS $\RKHS$ is, up to its closure in $\Lii$, invariant subspace of the Koopman operator $\Koop$, then finite rank non-defective operators on $\RKHS$ approximate arbitrarily well the restriction of $\Koop$ onto $\RKHS$. 

\propRisk*
\begin{proof}
Let us start by observing that $\TZ_\im\in\HS{\RKHS,\Lii}$, according to the spectral theorem for positive self-adjoint operators, has an SVD, i.e. there exists at most countable positive sequence $(\sigma_j)_{j\in J}$, where $J:=\{1,2,\ldots,\}\subseteq\N$, and ortho-normal systems $(\ell_j)_{j\in J}$ and $(h_j)_{j\in J}$ of $\cl(\range(\TZ_\im))$ and $\Ker(\TZ_\im)^\perp$, respectively, such that $\TZ_\im h_j = \sigma_j \ell_j$ and $\TZ_\im^* \ell_j = \sigma_j h_j$, $j\in J$.  

Now, recalling that $[\![\cdot]\!]_r$ denotes the $r$-truncated SVD, i.e. $[\![\TZ_\im]\!]_r = \sum_{j\in [r]}\sigma_j \ell_j\otimes h_j$,  since $\hnorm{\TZ_\im - [\![\TZ_\im]\!]_r }^2 = \sum_{j>r}\sigma_j^2$, for every $\delta>0$ there exists $r\in\N$ such that $\hnorm{\TZ_\im - [\![\TZ_\im]\!]_r } < \delta/3$.

Next, since $\range(\TZ_\im)\subseteq \cl(\range(\TS_\im))$, for every $j\in[r]$, we have that ${\ell}_j\in \cl(\range(\TZ_\im))\subseteq \cl(\range(\TS_\im))$, which implies that there exists ${g}_j\in\RKHS$ s.t. $\norm{{\ell}_j-\TS_\im {g}_j}\leq\frac{\delta}{3 r}$, and, denoting $B_r:=\sum_{j\in[r]}   \sigma_{j} {g}_j \otimes {h}_j$ we conclude $\hnorm{[\![\TZ_\im]\!]_r - \TS_\im B_r}\leq \delta / 3$.

Finally we recall that the set of non-defective matrices is dense in the space of matrices~\cite{TrefethenEmbree2020}, implying that the set of non-defective rank-$r$ linear operators is dense in the space of rank-$r$ linear operators on a Hilbert space. Therefore, there exists a non-defective $G\in\HSr$ such that $\hnorm{G - B_r }<\delta/ (3 \sigma_1(\TS_\im) )$. So, we conclude 
\begin{equation*}
\hnorm{\TZ_\im-\TS_\im G} \leq \hnorm{\TZ_\im- [\![ \TZ_\im ]\!]_r} + \hnorm{ [\![ \TZ_\im ]\!]_r - \TS_\im B_r} + \hnorm{\TS_\im(G-B_r) } = \delta.
\end{equation*}
\end{proof}

As a consequence of the previous result, we see that if $\RKHS$ (as a subspace of $\Lii$) is spanned by finitely many Koopman eigenfunctions, we have that $\TZ_\im$ can be approximated arbitrarily well. In practice such an assumption is not easy to check. On the other hand, for universal kernels we have that $\range(\TZ_\im)\subseteq\Lii=\cl(\range(\TS_\im))$, and hence we can learn $\TZ_\im$ arbitrarily well. 

We end this section with a brief discussion of the well-specified and misspecified cases mentioned in Rem.~\ref{rem:well_mis_spec} and prove the claim of Rem.~\ref{rem:self-adjoint} in the proposition that follows. To discuss this, we first introduce the following Tikhonov regularized version of problem \eqref{eq:KRP},
\begin{equation}\label{eq:Koopman_regresssion_Tikhonov}
\min_{\Estim\in\HS{\RKHS}} 
\Risk(\Estim)+\reg \hnorm{\Estim}^2,~~\reg > 0
\end{equation}
and note, by strong convexity, that its unique  solution is given by $\RKoop=(\TS_\im^*\TS_\im+\reg \Id_{\RKHS})^{-1} \TS_\im^*\TZ_\im$. 

{\em The well-specified case:} There exists $\HKoop\in\HS{\RKHS}$ such that $\TZ_\im = \TS_\im \HKoop$. In this case, $\RKHS$ as a subspace of $\Lii$ is an invariant subspace of $\Koop$, and, hence, $\HKoop\colon\RKHS\to\RKHS$ defines $\im$-a.e. the Koopman operator on the observable space $\RKHS$, i.e. $\HKoop f = \EE[f(X_{t+1})\,\vert\,X_t = \cdot]$ $\im$-a.e. for every $f\in\RKHS$. 
Moreover, in this case one has that $\HKoop = (\TS_\im^*\TS_\im)^{\dagger} \TS_\im^*\TZ_\im = \lim_{\reg\to0}\RKoop$, where $(\cdot)^\dagger$ denotes the densely defined Moore–Penrose pseudoinverse operator~\cite{SHSF2009}.  

{\em The misspecified case:} is when RKHS $\RKHS$ as a space of observables doesn't admit Hilbert-Schmidt $\im$-a.e. Koopman operator. This can clearly bring difficulties in learning $\Koop$ since, while one reduces $\ExRisk(\Estim)$, the HS norm $\hnorm{\Estim}$ may become progressively large. Note, however, that in this case it still might happen that operator norm $\norm{\Estim}$ stays bounded, and, even more, that $\RKHS$ as a subspace of $\Lii$ is an invariant set of $\Koop$, i.e. $\range(\TZ_\im)\subseteq\range(\TS_\im)$. As the following, somewhat surprising, result shows, this always happens when one learns self-adjoint Koopman operator via a universal kernel.


\begin{proposition}\label{prop:self-adjoint}
If the Markov process is reversible and $\range(\TZ_\im)\subseteq \cl(\range(\TS_\im))$, then there exists a bounded linear operator $\HKoop\colon\RKHS\to\RKHS$ such that $\norm{\HKoop}\leq1$, and for every $h\in\RKHS$, one has  $(\HKoop h)(x) = \EE[h(X_{t+1})\,\vert\,X_t = x]$ $\im$-almost everywhere. Moreover, for any $\reg>0$ it holds that $\norm{\TZ_\im-\TS_\im\RKoop}\leq\sqrt{\reg}$.
\end{proposition}

\begin{proof}
Let us start by observing that $\TS_\im\in\HS{\RKHS,\Lii}$, according to the spectral theorem for positive self-adjoint operators, has an SVD, i.e. there exists at most countable positive sequence $(\sigma_j)_{j\in J}$, where $J:=\{1,2,\ldots,\}\subseteq\N$, and ortho-normal systems $(\ell_j)_{j\in J}$ and $(h_j)_{j\in J}$ of $\cl(\range(\TS_\im))$ and $\Ker(\TS_\im)^\perp$, respectively, such that $\TS_\im h_j = \sigma_j \ell_j$ and $\TS_\im^* \ell_j = \sigma_j h_j$, $j\in J$. 

Using the above, we first prove that there exists a positive real non-increasing sequence $(\reg_{n})_{n\in\N}$ such that $\lim_{n\to\infty}\reg_n = 0$ and $\lim_{n\to\infty}\norm{\TZ_\im-\TS_\im \RKoop}=0$. To that end,  let $P$ and $Q$ denote orthogonal projectors in $\Lii$ onto $\cl(\range(\TS_\im))$ and in $\RKHS$ onto $\Ker(\TS_\im)^\perp$, respectively, i.e. $P = \sum_{j\in J}\ell_j\otimes\ell_j$ and $Q = \sum_{j\in J} h_j\otimes h_j$.  So, for every $\reg>0$ we have
\begin{equation*}
\norm{\TZ_\im - \TS_\im \RKoop}   = \norm{P \TZ_\im - \TS_\im \RKoop }
 = \norm{(P  - \TS_\im (\TS_\im^*\TS_\im+\reg I_\RKHS)^{-1}\TS_\im^*)\TZ_\im }, 
\end{equation*}
where the first equality is due to the fact that $\range(\TZ_\im)\subseteq \cl(\range(\TS_\im))$. Moreover, we have that $\TS_\im = \sum_{j \in J} \sigma_j \ell_j\otimes h_j$, and, hence,
\begin{eqnarray*}
P -  \TS_\im (\TS_\im^*\TS_\im+\reg I_\RKHS)^{-1}\TS_\im^* = \sum_{j \in J} \frac{\reg}{\reg+\sigma_j^2} \ell_j\otimes \ell_j \preceq I_{\Lii},\text{ and } \\
\TS_\im^*(P -  \TS_\im (\TS_\im^*\TS_\im+\reg I_\RKHS)^{-1}\TS_\im^*)^2 \TS_\im = \sum_{j \in J} \frac{\reg^2 \sigma_j^2}{(\reg+\sigma_j^2)^2} h_j\otimes h_j \preceq \reg \,Q.
\end{eqnarray*}
imply $\norm{P -  \TS_\im (\TS_\im^*\TS_\im+\reg I_\RKHS)^{-1}\TS_\im^*}\leq1$ and $\norm{(P -  \TS_\im (\TS_\im^*\TS_\im+\reg I_\RKHS)^{-1}\TS_\im^*) \TS_\im}\leq\sqrt{\reg}$, respectively.

Now, since $\range(\TZ_\im)\subseteq \cl(\range(\TS_\im))$, according  to Prop.~\ref{prop:hypothesis_space}, for every $n\in\N$ there exists a finite rank operator $B_n\colon \RKHS\to\RKHS$ such that $\norm{\TZ_\im - \TS_\im B_n}\leq 1/n$. Thus, denoting 
$Q_j:=\sum_{i\in[j]}h_i\otimes h_i$, $j\in J$, we have that for every $n\in\N$ there exists $j_n\in J$ such that 
\[
\norm{\TS_\im(Q-Q_{j_n})B_n} \leq \norm{\TS_\im(Q-Q_{j_n}}\norm{B_n} = \norm{\TS_\im - [\![\TS_\im]\!]_{j_n}}\norm{B_n}\leq 1/n,
\]
and, hence,  $\norm{\TZ_\im - \TS_\im Q_{j_n}B_n}\leq 2/n$. 

Therefore, for every $n\in\N$, there exists $j_n\in J$ such that for every $\reg>0$ it holds that
\begin{align*}
\norm{\TZ_\im - \TS_\im \RKoop} & = \norm{(P  - \TS_\im (\TS_\im^*\TS_\im+\reg I_\RKHS)^{-1}\TS_\im^*)(\TZ_\im \pm \TS_\im Q_{j_n} B_n)} \leq 2/n + \sqrt{\reg}\norm{Q_{j_n}B_n}.
\end{align*}

On the other hand, for the bounded operator $Q_{j_n}B_n$,  let $h\in\RKHS$ be such that $\norm{h}=1$ and $\norm{Q_{j_n}B_n h} = \norm{Q_{j_n}B_n}$. So, since $Q_{j_n}B_n h\in\Ker(\TS_\im)^\perp$,
\[
\frac{\norm{\TS_\im Q_{j_n}B_n}}{\norm{Q_{j_n}B_n}} = \frac{\norm{\TS_\im Q_{j_n}B_n}\norm{h}}{\norm{Q_{j_n}B_n h}}  \geq \frac{\norm{\TS_\im Q_{j_n}B_n h}}{\norm{Q_{j_n}B_n h}}  = \frac{\norm{\TS_\im Q_{j_n} Q_{j_n}B_n h}}{\norm{Q_{j_n}B_n h}} \geq \sigma_{\min}^+(\TS_\im Q_{j_n}) = \sigma_{j_n},
\]
and, thus, 
$ \norm{Q_{j_n}B_n}\leq \norm{ \TS_\im Q_{j_n}B_n} / \sigma_{j_n} \leq (\norm{\TZ_\im}+2/n) /  \sigma_{j_n}$. So, defining a sequence $\reg_n:=\frac{1}{n^2} \sigma_{j_n}^2$, we obtain
\[
\norm{\TZ_\im - \TS_\im \Estim_{\reg_n}}\leq \frac{1}{n}\left(\norm{\TS_\im} +  \frac{2}{n}\right) + \frac{2}{n},
\]
which converges to zero as $n\to\infty$.

Next, since one has that $\TS_\im^*\TZ_\im$ is self-adjoint and that $\TS_\im^*\TZ_\im \preceq \TS_\im^*\TS_\im$, 
\[
\norm{\RKoop}^2 = \norm{(\TS_\im^*\TS_\im+\reg I_\RKHS)^{-1} (\TS_\im^ *\TZ_\im)^2(\TS_\im^*\TS_\im+\reg I_\RKHS)^{-1}}\leq 1,
\]
implying that $(\Estim_{\reg_n})_{n\in\N}$ is uniformly bounded sequence of bounded operators on $\RKHS$. 
Hence, we conclude that there exists a bounded operator $\HKoop\colon\RKHS\to\RKHS$ and a weakly-$\star$ convergent subsequence $(\Estim_{\reg_{n_k}})_{k\in\N}$ such that $\Estim_{\reg_{n_k}}\overset{\star}{\rightharpoonup} \HKoop$  and $\norm{\HKoop}\leq1$. Now, since $\TZ_\im-\TS_\im\Estim_{\reg_{n_k}}\overset{\star}{\rightharpoonup} \TZ_\im-\TS_\im\HKoop$, we conclude that 
\[
\norm{\TZ_\im-\TS_\im\HKoop} \leq \liminf_{k\to\infty}\norm{\TZ_\im-\TS_\im\Estim_{\reg_{n_k}}}=0.
\]

To conclude the proof, observe that 
\[
\norm{\TZ_\im - \TS_\im\RKoop} = \norm{(P  - \TS_\im (\TS_\im^*\TS_\im+\reg I_\RKHS)^{-1}\TS_\im^*)\TS_\im \HKoop} \leq \sqrt{\reg}\norm{\HKoop}\leq\sqrt{\reg},\;\reg>0.
\]

\end{proof}


\subsection{Approximating Koopman Mode Decomposition by DMD}\label{app:kmd-dmd}

In this section we prove results stated in Thm. \ref{thm:KMD_learning} and Rem. \ref{rem:spectral}. 

\thmKMDLearning*

\begin{proof}
Given $f\in\RKHS$, denote $g:= (\TZ_\im - \TS_\im\Estim) f$, and $g_i := (\TZ_\im - \TS_\im \Estim) \refun_i$, $i\in[r]$. Then, for every $t\geq1$ we have
$\Koopt{t} \TS_\im f   = \Koopt{t-1}\TZ_\im f =  \Koopt{t-1}\TS_\im \Estim f + \Koopt{t-1}g$. Hence, using $
\TS_\im \Estim f = \sum_{i=1}^r\lambda_i\gamma_i^f\TS_\im \refun_i$ and $\TZ_\im \refun_i = \lambda_i \TS_\im \refun_i + g_i$, $i\in[r]$,
we obtain
\begin{eqnarray*}
\Koopt{t} \TS_\im f   & = &~ \Koopt{t-1} \big( \sum_{i=1}^r\lambda_i\gamma_i^f\TS_\im \refun_i \big) +  \Koopt{t-1}g =  \Koopt{t-2} \big( \sum_{i=1}^r\lambda_i\gamma_i^f\TZ_\im \refun_i \big) +  \Koopt{t-1}g \\
& = & ~\Koopt{t-2} \big( \sum_{i=1}^r \lambda_i^2\gamma_i^f\TS_\im \refun_i\big) + \Koopt{t-2} \big( \sum_{i=1}^r \lambda_i \gamma_i^f g_i \big) +  \Koopt{t-1}g  \\
& = &~ \cdots \\
& = & \sum_{i=1}^r \lambda_i^t\gamma_i^f\TS_\im \refun_i  + \big(\sum_{k=0}^{t-2}\Koopt{k}\big) \big( \sum_{i=1}^r \lambda_i \gamma_i^f g_i \big) +  \Koop^{t-1}g.
\end{eqnarray*}
However,  having that $\Koopt{t-1} g = \Koopt{t-1} (\TZ_\im - \TS_\im \Estim) f$ and
\begin{equation*}
\sum_{i=1}^r \lambda_i \gamma_i^f g_i = \sum_{i=1}^r \lambda_i  g_i \scalarp{f,\overline{\lefun}_i} = \sum_{i=1}^r \lambda_i  (\TZ_\im - \TS_\im \Estim) \refun_i \scalarp{f,\overline{\lefun}_i} = (\TZ_\im - \TS_\im \Estim) \Estim f
\end{equation*}
we obtain
\begin{equation*}
\Koopt{t} \TS_\im f   -   \sum_{i=1}^r \lambda_i^t\gamma_i^f\TS_\im \refun_i  = \big(\sum_{k=0}^{t-2}\Koopt{k}\big) (\TZ_\im - \TS_\im \Estim) \Estim f +  \Koopt{t-1} (\TZ_\im - \TS_\im \Estim) f.
\end{equation*}
So, to conclude \eqref{eq:KMD_bound}, it suffices to recall that $\norm{\Koop} = 1$ and apply norm in $\Lii$
\begin{equation*}
\norm{\Koopt{t} \TS_\im f   -   \sum_{i=1}^r \lambda_i^t\gamma_i^f\TS_\im \refun_i } \leq \norm{\TZ_\im - \TS_\im \Estim}\big((t-1) \norm{\Estim f} + \norm{f}\big) 
\end{equation*}

We now prove \eqref{eq:bound_efun}. Since $g_i =  \TZ_\im \refun_i - \TS_\im(\lambda_i\,\refun_i) = \Koop (\TS_\im\refun_i) - \lambda_i(\TS_\im \refun_i) $, $i\in[r]$, we obtain that 
\begin{equation*}
\norm{(\Koop (\TS_\im\refun_i )- \lambda_i(\TS_\im \refun_i) } = \norm{g_i} \leq \norm{\TZ_\im - \TS_\im \Estim} \norm{\refun_i}.
\end{equation*}

However, since $\refun_i\in\range(\Estim)\setminus\{0\}$, there exists $h_i\in\Ker(\Estim)^\perp$ so that $ \refun_i = \Estim h_i $. 
Recalling that $C_\reg = C+\reg I_{\RKHS}$ is positive definite for $\reg>0$, we have that $\range(\Estim^*) = \range(\Estim^*C_\reg^{1/2})$, and, consequently, $\Ker(\Estim)^\perp = \Ker(C_\reg^{1/2} \Estim)^\perp$. Thus, since 
\begin{equation*}
\inf_{h\in  \Ker(C_\reg^{1/2} \Estim)^\perp }\frac{\norm{ C_\reg^{1/2} \Estim h}}{\norm{h}} = \sigma_{\min}^+(C_\reg^{1/2} \Estim) = \sigma_r(C_\reg^{1/2} \Estim),
\end{equation*}
we obtain that $ \norm{ C_\reg^{1/2} \refun_i} \geq \sigma_r(C_\reg^{1/2} \Estim)  \norm{h_i}$, which letting $\reg\to0$ implies that $\norm{\TS_\im \refun_i}\geq \sigma_r(\TS_\im \Estim) \norm{h_i}$. Hence, we derive $\norm{\refun_i}\leq \norm{\Estim}\norm{h_i}\leq \norm{\Estim} \norm{\TS_\im \refun_i} /  \sigma_r(\TS_\im \Estim)$,  which proves \eqref{eq:bound_efun}.

%
\end{proof}

In the following example we show that the bound \eqref{eq:bound_efun} w.r.t. arbitrary estimator is tight. 

\begin{example}\label{ex:OU}
As a specific instance of Exm.~\ref{ex:noisy_ds} is the equidistant sampling of Ornstein–Uhlenbeck process $X_{t+1} = F X_t +\noise_t$, where $F\in\R^{d\times d}$ and the noise $\noise_t$ is Gaussian. For simplicity, let $F=F^*$ with eigenvalues $\lambda_i$ in $]0,1[$, and let the noise be i.i.d. from $\mathcal{N}(0,I_d)$. It is well-known~\cite[Chapter 10.5]{Meyn1993}, that the invariant distribution is $\mathcal{N}(0,C)$, where 
$C = (I_d-F^2)^{-1}$. If the linear kernel is used, it is readily checked that the corresponding RKHS $\RKHS$ is a closed invariant subspace of $\Koop$ and, moreover, $\TZ_\im = \TS_\im \HKoop$, where $\HKoop$ is given by $F$. Now, consider the rank-$r$ estimator $\Estim = 2 [\![\HKoop]\!]_r$. Denoting $\beta_i = \lambda_i / \sqrt{1-\lambda_{i}^{2}}>0$, $i\in[n]$, we have that $\ExRisk(\Estim) = \norm{\TZ_\im - \TS_\im G} = \norm{(I_d-F^2)^{-1/2}(F - 2 [\![F]\!]_r)} =  \beta_1$, $\norm{G} = 2\norm{F}=2\lambda_1$ and $\sigma_r(\TS_\im G) = 2\norm{(I_d-F^2)^{-1/2} [\![F]\!]_r} = 2\beta_r$. Therefore, for every  eigenpair $(\lambda_i,v_i)$ of $F$, $i\in[r]$, we have $\refun_i = \scalarp{v_i,\cdot}\in\RKHS$ and $\Estim\refun_i = 2\lambda_i \refun_i$, so, consequently, $\norm{(\Koop (\TS_\im\refun_i )- 2\lambda_i(\TS_\im \refun_i) } = \lambda_i\norm{\TS_\im \refun_i}$. Therefore, assuming that $\lambda_1=\ldots = \lambda_{r}$, for this estimator, we attain equality in \eqref{eq:bound_efun} for all $i\in[r]$.
\end{example}

We conclude this section with the result that links the introduced risk to two key concepts of eigenvalue perturbation analysis. First is Stewart's definition of \textit{spectral separation} between two bounded operators on (possibly different) Hilbert spaces, see~\cite{Stewart1971},
\begin{equation}\label{eq:separation}
\sep(A,B):=\min_{\norm{C}_{\rm HS}=1} \norm{AC - CB}_{\rm HS},
\end{equation}
and the second is pseudospectrum of bounded linear operators, see~\cite{TrefethenEmbree2020}, 
\begin{equation}\label{eq:pseudospectra}
\text{\rm Sp}_{\varepsilon}(A):=\bigcup_{\norm{B}\leq\varepsilon}\Spec(A+B) = \{z\in\C\,\vert\, \norm{(A-zI)^{-1}}^{-1}\leq \varepsilon\},
\end{equation}
with the convention $\norm{(A-zI)^{-1}}^{-1}=0$ whenever $z$ is not in the resolvent set of $A$, i.e. $z\in\Spec(A)$. 



\begin{corollary}\label{cor:spectra_learning_gen}
If the eigenfunctions of $\Estim\in\HSr$ are not $\im$-a.e. zero, then
\begin{equation}\label{eq:separation_risk}
\sep(\Koop,\Estim)\leq \sqrt{\ExRisk(\Estim)}\norm{\TS_\im}_{\rm HS},
\end{equation}
and 
\begin{equation}\label{eq:pseudospectra_risk}
\Spec(\Estim) \subseteq \text{\rm Sp}_{\varepsilon}(\Koop), \quad{ for }\; \varepsilon = \norm{\TZ_\im - \TS_\im \Estim} \norm{\Estim} / \sigma_{r}(\TS_\im \Estim).
\end{equation}

Consequently, if $\Koop$ is normal, then for every $\lambda \in\Spec(\Estim)$ there exists $\lambda_{\im}\in\Spec(\Koop)$ such that $\abs{\lambda_\im-\lambda}\leq \norm{\TZ_\im - \TS_\im \Estim} \norm{\Estim}  / \sigma_{r}(\TS_\im \Estim)$. If additionally $\Estim$ is normal, then $\abs{\lambda_\im-\lambda}\leq \sqrt{\ExRisk(\Estim)}\norm{\TS_\im}_{\rm HS}$.
\end{corollary}
\begin{proof}
Inequality \eqref{eq:separation_risk} is a direct consequence of the definition of the separation. On the other hand, \eqref{eq:pseudospectra_risk} follows immediately from  \eqref{eq:bound_efun} and the fact that 
\begin{equation}\label{eq:resolvent}
\norm{(z I_{\Lii}-\Koop)^{-1}}_{\Lii}^{-1}  = \min_{f\in\Lii}\frac{\norm{\Koop f - z f}}{\norm{f}},\quad z\in{\rm Res}(\Koop),
\end{equation}
by taking $\TS_\im \refun_i\neq0$ in place of $f$ and $\lambda_i$ in place of $z$.

Now, using that, see~\cite{TrefethenEmbree2020}, for any normal operator $A$
\begin{equation}\label{eq:pseudospectra_1}
\min_{z'\in\Spec(A)}\abs{z-z'}\leq\varepsilon, \quad z\in {\rm Sp}_\varepsilon(A),
\end{equation}
and that for any two normal operators $A$ and $B$ 
\begin{equation}\label{eq:separation_1}
\sep(A,B) = \min\{\abs{z-z'} \;\vert\; z\in\Spec(A),\, z'\in\Spec(B)\},
\end{equation}
the last two statements follow.
\end{proof}

\begin{remark}\label{rem:normal_koopman}
In App.~\ref{app:background} we discussed two important cases in which Koopman operator on $\Lii$ is normal, namely the case of deterministic dynamical systems when $\Koop$ is unitary (cf. Ex.~\ref{ex:deterministic}) and the case of time reversible Markov chains, i.e. when $\Koop$ is self-adjoint (cf. Rem.~\ref{rem:reversible_dynamics}). In such cases, the previous result motivates one to consider normal estimators of the Koopman operator.
\end{remark}

\subsection{Duality Between KOR
and CME}\label{app:kor-cme}


In this section we clarify Rem. \ref{rem:kor_cme} on the relationship between conditional mean embeddings (CME) and Koopman operator regression (KOR). 
Recalling the definition of CME in \eqref{eq:CME} and the restriction of the Koopman operator on $\RKHS$, $\TZ_\im := \Koop\TS_\im$, it is easy to see that for every $f\in\RKHS$ it holds 
\begin{eqnarray}\label{eq:duality_1}
(\TZ_\im f)(x)=\EE[f(X_{t+1})\,\vert\, X_t=x] = \EE[\scalarp{f,\phi(X_{t+1})},\vert\, X_t=x] = \scalarp{f,\CME(x)}.
\end{eqnarray}
The CME of the Markov transition kernel $\transitionkernel$ is therefore just the Riesz representation of the functional evaluating the Koopman operator restricted to $\RKHS$.

We now prove the identity \eqref{eq:true_risk_cme} in the main text.
\begin{proposition}\label{prop:risk_cme}
For every $\Estim\in\HS{\RKHS}$ the risk \eqref{eq:true_risk} can be equivalently written as
\begin{equation}\label{eq:risk_cme}
\underbrace{ \EE_{(x,y)\sim \rho} \norm{\phi(y) - \Estim^*\phi(x)}^2}_{\Risk(\Estim)} =  \underbrace{\EE_{(x,y)\sim\rho}\norm{\CME(x) - \phi(y) }^2}_{\IrRisk} +  \underbrace{\EE_{x\sim \im} \norm{\CME(x) - \Estim^*\phi(x)}^2}_{\ExRisk(\Estim)}.
\end{equation}
\end{proposition}

\begin{proof}
Starting from \eqref{eq:true_risk} and using the reproducing property we obtain
\begin{align*} 
\Risk(\Estim) & = \sum_{i\in\N} \EE_{(x,y)\sim \rho}[h_i (y)- (\Estim h_i)(x)]^2 =  \sum_{i\in\N} \EE_{(x,y)\sim \rho} [\scalarp{h_i, \phi(y)}- \scalarp{\Estim h_i, \phi(x)}]^2 \\
& =  \sum_{i\in\N}  \EE_{(x,y)\sim \rho} [\scalarp{h_i, \phi(y)}- \scalarp{h_i, \Estim^*\phi(x)}]^2 =  \sum_{i\in\N}  \EE_{(x,y)\sim \rho} \scalarp{h_i, \phi(y)-\Estim^* \phi(x)}^2 \\
& =  \EE_{(x,y)\sim \rho}\sum_{i\in\N} \scalarp{h_i, \phi(y)-\Estim^* \phi(x)}^2 = \EE_{(x,y)\sim \rho}\norm{\phi(y)-\Estim^* \phi(x)}_{\RKHS}^2.
\end{align*}
By the reproducing properties of
$(\TZ_\im h_i)(x) = \scalarp{h_i,\CME(x)}$ and $(\TS_\im h_i)(x) = \scalarp{h_i,\phi(x)}$ and Proposition~\ref{prop:risk} we have that
\begin{align*} 
\ExRisk(\Estim) & =  \hnorm{\TZ_\im - \TS_\im\Estim}^2 = \sum_{i\in\N} \norm{\TZ_\im h_i- \TS_\im \Estim h_i}^2 =  \sum_{i\in\N} \EE_{x\sim\im}[ \abs{(\TZ_\im h_i)(x)- (\TS_\im \Estim h_i)(x)}^2] \\
& =  
\EE_{x\sim\im}\Big[ \sum_{i\in\N} \scalarp{h_i,\CME(x) - \Estim^* \phi(x)}^2\Big] = \EE_{x\sim\im}\Big[ \norm{\CME(x)-\Estim^*\phi(x)}^2\Big].
\end{align*} 
Moreover, since $\TS_\im^*\TS_\im = \EE_{y\sim\im}[\phi(y)\otimes\phi(y)]$ and $\TZ_\im^*\TZ_\im = \EE_{x\sim\im}[\CME(x)\otimes\CME(x)]$, we have
\[ 
\EE_{(x,y)\sim\rho}\norm{\phi(y)-\CME(x)}^2  = \tr(\TS_\im^*\TS_\im) - 2 \EE_{(x,y)\sim\rho} \scalarp{\phi(y),\CME(x)} + \tr(\TZ_\im^*\TZ_\im),
\]
which, along with Prop.~\ref{prop:risk} and the identity $\EE_{(x,y)\sim\rho} \scalarp{\phi(y),\CME(x)} = \EE_{x\sim\im} \EE [ \scalarp{\phi(Y),\CME(x)}\,\vert\, X = x] = \EE_{x\sim\im} \scalarp{\CME(x),\CME(x)} = \tr(\TZ_\im^*\TZ_\im)$ completes the proof.
\end{proof}
Prop.~\ref{prop:risk_cme} implies that $\Estim_\star$ is a solution of the KOR problem \eqref{eq:KRP} if and only if $\Estim_\star^*$ is a solution of the CME regression problem $\min_{\Estim} \EE_{(x,y)\sim \rho} \norm{\phi(y) - \Estim\, \phi(x)}^2$. In this sense the Koopman regression problem is \textit{dual} to learning CME of the Markov transition kernel $p$. 

Moreover, from the perspective of CME, well-specified case is identified by $\CME(\cdot) = \HKoop^*\phi(\cdot)$, i.e. it is the case when regression operator $\CME$ belongs to the vector-valued RKHS $\mathcal{G}$ defined by the operator-valued kernel $g(x,x'):=k(x,x')\Id_{\RKHS}$. This vector-valued RKHS is isometrically isomorphic to $\HS{\RKHS}$, where the isomorphism is given by $\HS{\RKHS}\ni A\longleftrightarrow A\phi(\cdot)\in\mathcal{G}$, see~\cite[Ex.3.6(i)]{CDTU2010}. On the other hand, the misspecified case is simply when $\CME\not\in\mathcal{G}$.

\section{Empirical Risk Minimization}\label{app:erm}

In this section we provide details on computing the estimators of the Koopman operator. For convenience, we denote the regularized risk  by 
\begin{equation}\label{eq:empirical_risk_reg}
\widehat{\Risk}^\reg(\Estim) :
    = \hnorm{ \EZ - \ES\Estim }^{2} + \reg\hnorm{\Estim}^{2}, \qquad \Estim \in \HS{\mathcal{H}}.
\end{equation}

\subsection{Computation of the Estimators}\label{app:comp_estim}

In Theorem~\ref{thm:R3_1} we derive the closed form solution of \eqref{eq:HSnorm_RRR_1} and in Theorem~\ref{thm:R3_2} we formulate it in a numerically computable representation. In Theorem~\ref{thm:PCR} we show the same for the PCR estimator, highlighting its equivalence to the kernel DMD algorithm~\cite{Kutz2016}.

\begin{theorem}\label{thm:R3_1}
The optimal solution of problem \eqref{eq:HSnorm_RRR_1} is given by $\EEstim_{r, \reg} = \ECx_\reg^{-\frac{1}{2}} [\![ \ECx_\reg^{-\frac{1}{2}}\ECxy ]\!]_r $. Moreover, $\widehat{\Risk}^\reg(\EEstim_{r,\reg}) = \tr(\ECy)  - \sum_{i=1}^{r} \sigma_i^2$,
where $\sigma_1\geq\cdots\geq\sigma_r$ are leading singular values of $\ECx_\reg^{-\frac{1}{2}}\ECxy$.
\end{theorem}

\begin{proof} Start by observing that, according to  \eqref{eq:empirical_risk},
\begin{align*}
\widehat{\Risk}^\reg(\Estim) & = \frac{1}{n}\sum_{i=1}^n \left \Vert \phi(y_{i}) - \Estim^{*}\phi(x_{i})\right \Vert^{2} + \reg\hnorm{\Estim}^{2}  \\
& =  \frac{1}{n}\sum_{i=1}^n \tr( \phi(y_i)\otimes \phi(y_i))  - 2\scalarp{\phi(y_i),\Estim^*\phi(x_i)} + \tr(\Estim\Estim^*\phi(x_i)\otimes \phi(x_i)) + \reg \tr(\Estim\Estim^*) \\
& =  \tr(\ECy) + \tr(\Estim\Estim^* \ECx_{\reg}) - 2\tr(\Estim^* \ECxy) =
\tr(\ECy) - \hnorm{\ECx_\reg^{-\frac{1}{2}}\ECxy}^2 + \hnorm{\ECx_\reg^{\frac{1}{2}}\Estim  - \ECx_\reg^{-\frac{1}{2}}\ECxy}^2.
\end{align*}
The last equality follows from simple algebra after adding and subtracting the term $\hnorm{\ECx_\reg^{-\frac{1}{2}}\ECxy}^2$.
 
We now focus on the last term of the previous equation, the only one entering the minimization. We have
\begin{equation}\label{eq:RRR_minimization}
    \hnorm{ [\![\ECx_\reg^{-\frac{1}{2}}\ECxy]\!]_r  - \ECx_\reg^{-\frac{1}{2}}\ECxy}^2 = \min_{B \in \HSr} \hnorm{B  - \ECx_\reg^{-\frac{1}{2}}\ECxy}^2 \leq \min_{G \in \HSr} \hnorm{\ECx_\reg^{\frac{1}{2}}\Estim  - \ECx_\reg^{-\frac{1}{2}}\ECxy}^2.
\end{equation}
The equality above comes from the Eckart–Young–Mirsky theorem, while the inequality from the fact that $\Estim \in \HSr \implies B := \ECx_\reg^{\frac{1}{2}}\Estim \in \HSr$. From~\eqref{eq:RRR_minimization} we conclude that $\EEstim_{r, \reg} = \ECx_\reg^{-\frac{1}{2}} [\![ \ECx_\reg^{-\frac{1}{2}}\ECxy ]\!]_r$ minimizes $\widehat{\Risk}^\reg$. The same theorem also guarantees that $\norm{\ECx_\reg^{\frac{1}{2}} \EEstim_{r,\reg}  - \ECx_\reg^{-\frac{1}{2}}\ECxy} = \sum_{i=r+1}^{\infty}\sigma_i^2$, hence 
\begin{equation*}
 \widehat{\Risk}^\reg(\EEstim_{r,\reg}) = \tr(\ECy) - \sum_{i=1}^{\infty}\sigma_i^2 + \sum_{i=r+1}^{\infty}\sigma_i^2 =  \tr(\ECy) - \sum_{i=1}^{r}\sigma_i^2.
\end{equation*}
\end{proof}
While the previous theorem provides a method to compute the RRR estimator when the RKHS is finite-dimensional (in fact, an efficient one when the number of features is smaller than the number of samples), the following result shows how one can compute RRR for infinite-dimensional RKHSs.

\begin{theorem}\label{thm:R3_2} If $U_r = [u_1\, \vert \ldots \,\vert u_r]\in\R^{n\times r} $ is such that $(\sigma_i^2,u_i)$ are the solutions of the generalized eigenvalue problem 
\begin{equation}\label{eq:GEV} \Ky \Kx  u_i = \sigma_i^2 K_\reg u_i\quad\text{ normalized such that }\quad u_i^\top \Kx  K_\reg u_i = 1, \quad i \in[r] 
\end{equation} 
and $V_r = \Kx  U_r$, then the optimal solution of  \eqref{eq:HSnorm_RRR_1}  is given by $\EEstim_{r,\reg} =  \ES ^* U_r V_r^\top \EZ $. Moreover, we have
\begin{equation}\label{eq:error} \widehat{\Risk}^\reg(\EEstim_{r,\reg}) = \tr(\Ky)  - \sum_{i=1}^{r} \sigma_i^2\;  \text{ and }\; \widehat{\Risk}(\EEstim_{r,\reg}) = \tr\Big( \big( I - \Kx U_r V_r^\top - \reg \Kx (U_r V_r^\top)^2 \big)  \Ky\Big). 
\end{equation} 
\end{theorem}

\begin{proof} Start by observing that, according to Thm~ \ref{thm:R3_1}, $\EEstim_{r,\reg}$ is obtained from the truncated SVD of the operator $(\ECx+ \reg I_{\RKHS})^{-\frac{1}{2}}\ECxy = (\ES ^*\ES + \reg I_{\RKHS})^{-\frac{1}{2}}\ES ^*\EZ  = \ES ^*(\ES \ES ^*+ \reg I_{\RKHS})^{-\frac{1}{2}}\EZ  =  \ES ^*K_\reg^{-\frac{1}{2}}\EZ  $. Its leading singular values $\sigma_1\geq\ldots\geq\sigma_r$ and the corresponding {\em left} singular vectors $g_1,\ldots,g_r\in\RKHS$ are obtained by solving the eigenvalue problem
\begin{equation}\label{eq:evp_1}
 \left(\ES ^*K_\reg^{-\frac{1}{2}}\EZ\right)\left(\ES ^*K_\reg^{-\frac{1}{2}}\EZ\right)^{*} g_i  = \sigma_i^2 g_i,\; i\in[r].
\end{equation}
From the above equation, clearly $g_i\in\range(\ES ^* K_\reg^{-\frac{1}{2}}) = \range(\ES ^* K_\reg^{\frac{1}{2}})$, and we can represent the singular vectors as $g_i = \ES ^* K_\reg^{\frac{1}{2}} u_i$ for some $u_i\in\R^n$, $i\in[r]$. Therefore, substituting $g_i = \ES ^* K_\reg^{\frac{1}{2}} u_i$ in~\eqref{eq:evp_1} and simplifying one has the finite-dimensional eigenvalue equation
\begin{equation}\label{eq:evp_2}
\Ky\Kx  u_i = \sigma_i^2 K_\reg u_i,\; i\in[r].
\end{equation}
Solving~\eqref{eq:evp_2} and using that $g_i = \ES ^* K_\reg^{\frac{1}{2}} u_i$, one obtains  $(\sigma_i^2,g_i)$, $i\in[r]$, the solutions of the eigenvalue problem \eqref{eq:evp_1}. In order to have properly normalized $g_i$, it must hold for all $i\in[r]$ that
\begin{equation}\label{eq:eigenvector_normalization}
1 =g_i^*g_i = u_i^\top K_\reg^{\frac{1}{2}} \ES  \ES ^* K_\reg^{\frac{1}{2}}  u_i = u_i^\top  \Kx  K_\reg u_i.
\end{equation}

Now, the subspace of the leading left singular vectors is $\range( \ES ^* K_\reg^{\frac{1}{2}} U_r )$. As the columns of $U_r$ are properly normalized according to~\eqref{eq:eigenvector_normalization}, the orthogonal projector onto the range of $\ES ^* K_\reg^{\frac{1}{2}} U_r$ is given by $\Pi_r := \ES ^* K_\reg^{\frac{1}{2}} U_r U_r^\top K_\reg^{\frac{1}{2}} \ES $. We therefore have that $[\![ \ECx_\reg^{-\frac{1}{2}}\ECxy ]\!]_r =  \Pi_r \ECx_\reg^{-\frac{1}{2}}\ECxy = 
\ES ^* K_\reg^{\frac{1}{2}} U_r U_r^\top \Kx \EZ $. Thus, defining $V_r := \Kx  U_r$, we conclude that 
\begin{equation*}
    \EEstim_{r,\reg} = \ECx_\reg^{-\frac{1}{2}}  \ES ^* K_\reg^{\frac{1}{2}} U_r V_r^\top \EZ  = 
\ES ^*U_r V_r^\top \EZ. 
\end{equation*} 

\medskip

To conclude the proof we have to evaluate the error $\widehat{\Risk}(\EEstim_{r,\reg})$. We notice that $\tr(\ECy)=\tr(\Ky )$ and that
\begin{align*}
\widehat{\Risk}(\EEstim_{r,\reg}) 
= \tr(\Ky) - 2\tr(V_r V_r^\top \Ky ) + \tr(V_r V_r^\top V_r V_r^\top  \Ky) = \tr\Big( \big( I - \Kx U_r V_r^\top - \reg \Kx (U_r V_r^\top)^2 \big)  \Ky\Big).
\end{align*} 
Here, along with some simple algebric manipulations, we have used $U_r^\top \Kx (\Kx+\reg I) U_r = I $, i.e. $V_r^\top V_r + \reg V_r^\top U_r = I$. 
\end{proof}

\begin{remark}\label{rem:empirical_risk_krr}
Since for $r=n$ the estimators RRR and KRR coincide, the previous result implies that the empirical risk for the KRR estimator can be written as 
\begin{equation}\label{eq:krr_empirical_risk}
\widehat{\Risk}(\EEstim_{\reg}) = \tr\Big( \big( I - \Kx \Kx_\reg^{-1} + \reg \Kx \Kx_\reg^{-2} \big)  \Ky\Big) = \reg^2\tr\big( \Kx_\reg^{-2} \Ky \big).
\end{equation}
\end{remark}

As discussed above, see also \cite{Kutz2016}, for the choice of linear kernel PCR estimator $\EEstim^{{\rm PCR}}_r =  [\![ \ECx]\!]_r^{\dagger}\ECxy$ is known as DMD, while for the finite-dimensional (nonlinear) kernels it is known as extended EDMD. In these cases previous formula gives also a practical way to compute it. On the other hand, for infinite-dimensional kernels, PCR is known as kernel DMD, and in this case its practical computation can be done using the following result. 

\begin{theorem}\label{thm:PCR} The PCR estimator $\EEstim^{{\rm PCR}}_r =  [\![ \ECx]\!]_r^{\dagger}\ECxy$ can be equivalently written as $\EEstim^{{\rm PCR}}_r = \ES^{*} U_r V_r^\top\EZ$,
where $[\![\Kx]\!]_r^{\dagger} = V_r \Sigma_r V_r^\top$ is $r$-trunacted SVD and $U_r = V_r \Sigma_r^\dagger$.  Moreover, it holds that 
\begin{equation}\label{eq:error_pcr}  
\widehat{\Risk}(\EEstim_r^{\rm PCR}) = \tr(( I_n - \Kx U_r V_r^\top )\Ky). 
\end{equation} 
\end{theorem}

\begin{proof}
Without loss of generality assume that $\rank(\ECx)\geq r$. As for the proof of Theorem~\ref{thm:R3_2}, the leading $r$ singular vectors $(g_i)_{i\in[r]}$ of $\ECx$ can be written in the form $g_i = \ES^*v_i / \sqrt{\sigma_i}$, where $v_i$ are the eigenvectors corresponding to the $r$ leading eigenvalues $\sigma_i$ of $\Kx$
. Then, it readily follows that $[\![\ECx]\!]_r^{\dagger} \ECxy = \ES^* V_r \Sigma_r^{-2} V_r^\top \ES \ES^*\EZ$. We therefore conclude that
\begin{equation*}
    \EEstim^{{\rm PCR}}_r = \ES^* V_r \Sigma_r^{-2} (\Kx V_r)^\top \EZ = \ES^* V_r \Sigma_r^{-1} V_r^\top \EZ = \ES^* U_r V_r^\top \EZ
\end{equation*}
Finally, since $\widehat{\Risk}( \EEstim^{{\rm PCR}}_r ) = \tr(\Ky) - 2\tr(V_r V_r^\top \Ky ) + \tr(V_r V_r^\top V_r V_r^\top  \Ky)$, using that $V_rV_r^\top$ is orthogonal projector and $\Kx U_r = V_r$ we obtain \eqref{eq:error_pcr}.
\end{proof}

\subsection{Mode Decomposition and Prediction}\label{app:comp_dmd}

In this section we show how an estimator of the Koopman operator can be used to predict future states of the system and how its mode decomposition can be evaluated. We will address a slightly more general setting than the one presented in Sec.~\ref{sec:erm}, that is we allow for {\em vectorial} observables $f = (f_\ell)_{\ell=1}^m \in \RKHS^m$ for which the action of the Koopman operator is naturally extended as $\Koop f = (\Koop f_\ell)_{\ell\in[m]}$. To that end, given the data $\Data$ and a vector valued observable $f = (f_\ell)_{\ell=1}^m \in \RKHS^m$ we denote the observable evaluated along the data points as $\Gamma^f = \big[f(y_{1})\,\vert\, \ldots\,\vert\, f(y_{n}) \big]\in \R^{m\times n}$. 
\begin{remark}\label{rem:states_and_non_rkhs_obesrvables}
If $\X\subseteq\R^d$, we will argue that an important observable of the system is given by the identity function ${\rm Id}:\X \to \X$. If the projection onto the $i$-th component is a function belonging to $\RKHS$ for all $i \in [d]$, then ${\rm Id} \in \RKHS^{d}$ and Thm.~\ref{thm:KMD_learning} holds for this specific observable. While this may not hold in general, note that for every kernel we can take the sum with a linear kernel to obtain RKHS that contains ${\rm Id}$.
\end{remark}

\paragraph{Prediction.} Each empirical estimator $\EEstim = \ES^* W \EZ$ allows one to estimate a future state given a starting point.
 According to the bound~\eqref{eq:KMD_bound} in Thm.~\ref{thm:KMD_learning} we obtain that given $f = (f_\ell)_{\ell=1}^m \in \RKHS^m$, if $x$ is the current state of the Markov process, the expected value of $f$ at the next iteration
is approximated as $[\Koop f](x) = [\EEstim f](x) + {\rm err}^f(x)$, i.e.
\begin{equation}
\EE[ f(X_{t + 1})\, \vert\, X_t=x] = \sum_{j=1}^n \beta_j^f k(x_j,x) + \sqrt{\ExRisk(\EEstim)}\;{\rm err^f}(x),
\end{equation}
where $\beta^f = \tfrac{1}{n}\Gamma^f W^\top\in\R^{m\times n}$ and ${\rm err^f}\in(\Lii)^{m}$ such that $\norm{({\rm err^f})_\ell}\leq \norm{f_\ell}$, $\ell\in[m]$. If $f = {\rm Id}$, we obtain a prediction of the future state $\EE[ X_{t + 1}\, \vert\, X_t=x]$.

\paragraph{Modal Decomposition \& Forecasting.} A more general instance of prediction is given by forecasting through modal decomposition as showed~\eqref{eq:KMD_bound}. The main ingredient needed to forecast via mode decomposition is the spectral decomposition of the estimator $\EEstim$. We now prove a slightly more general version of Theorem~\ref{thm:DMD_theorem}, allowing us to compute the eigenvalue decomposition of $\EEstim$ numerically.

\begin{theorem}\label{thm:DMD_theorem_2}
Let $\EEstim = \ES^* U_r V_r^\top \EZ$, with $U_r, V_r \in \R^{n\times r}$. If $V_r^\top \Kyx U_r \in\R^{r\times r}$ is full rank and non-defective, the spectral decomposition $(\lambda_i,\lefun_i, \refun_i)_{i\in[r]}$ of $\EEstim$ can be expressed in terms of the spectral decomposition $(\lambda_i,\levec_i, \revec_i)_{i\in[r]}$ of $V_r^\top \Kyx U_r $.  Indeed, for all $i\in[r]$, one has   $\lefun_i = \EZ^*V_r \levec_i / \overline{\lambda}_i$ and $\refun_i = \ES^*U_r \revec_i$. In addition, for every $f\in\RKHS^m$ dynamic modes are $\gamma_i^f =   \Gamma^f (\levec_i ^*V_r^\top)^\top / (\lambda_i \sqrt{n})\in\C^m$.

\end{theorem}

\begin{proof}
First note that since in general $\EEstim$ is not self-adjoint, its eigenvalues may come in complex conjugate pairs. Hence, for $\xi_i\in\RKHS$ and $\psi_i\in\RKHS$ left and right eigenfunctions of $\EEstim$ corresponding to its eigenvalue $\lambda_i$, we have $\EEstim^*\lefun_i = \overline{\lambda}_i \lefun_i$ and $\EEstim\refun_i = \lambda_i \refun_i$, $i\in[r]$. To avoid cluttering, in the following we will only show the explicit calculation of the right eigenfunctions $\refun_{i}$. We stress, however, that the calculation of the left eigenfucntions $\lefun_{i}$ follows exactly the same arguments.

From $\range(\EEstim)\subseteq \range(\ES ^*)$ it follows that  for all $i\in[r]$, $\refun_i\in{\cal S} := \{\sum_{j=1}^n w_j\phi(x_j)\,\vert\, w\in\C^{n}\}$. Using \cite[Prop.~ 3.8]{MSKS2020}, we have that $\refun_i = \ES ^*\widehat{v}_i$, where $\widehat{v}_i\in\C^{n}\setminus\{0\}$ are eigenvectors of $ U_r V_r^\top \Kyx$. Since all the eigenvalues $\lambda_i$ we are considering are nonzero, the spectral decomposition of $U_r V_r^\top \Kyx$ is equivalent~\cite{SS1990} to
\begin{equation}\label{eq:EVP_right}
V_r^\top \Kyx U_r \revec_i = \lambda_{i} \revec_i \;\text{ and } \widehat{v}_i = U_r \revec_i, \quad i\in[r].
\end{equation}

Therefore  $\refun_i = \ES ^*U_r \revec_i$ are the right eigenfunctions of $\EEstim = \ES^* U_r V_r^\top \EZ$. With the same arguments we can show that $\lefun_i = \EZ ^*V_r \levec_i$ are the left eigenfunctions, $\levec_i$ being the leading eigenvectors of  $U_r^\top \Kxy V_r$. Re-normalizing $\lefun_j = \EZ ^*V_r \levec_j / \overline{\lambda}_j$ we obtain that for every $i,j\in[r]$
\begin{align*}
\scalarp{\refun_i, \overline{\lefun}_j} & =   \levec_j^*V_r^\top\EZ\ES^*U_r \revec_i  / \lambda_j = \levec_j^*V_r^\top\Kyx U_r \revec_i / \lambda_j =  (U_r^\top\Kxy V_r\levec_j)^* \revec_i / \lambda_j = \levec_j^* \revec_i = \delta_{ij},
\end{align*}
which assures that $(\lambda_i,\lefun_i, \refun_i)_{i\in[r]}$ is the spectral decomposition of $\EEstim$. Above we have assumed (without loss of generality) that $\levec_j^*\revec_i = \delta_{ij}$, $i,j\in[r]$, i.e. that the left and right eigenvectors of $V_r^\top \Kyx U_r$ are mutually orthonormal. 

Finally, since $\EEstim = \sum_{i\in[r]}\lambda_i \refun_i \otimes \overline{\lefun}_i$ we have that $\EEstim f_\ell  =  \sum_{i\in[r]}\lambda_i \refun_i \scalarp{f_\ell,\overline{\lefun}_i}$ and, consequently  
\begin{equation*}
\gamma_i^f =  (\scalarp{f_\ell, \overline{\lefun}_i})_{\ell\in[m]}  = (\scalarp{\levec_i^*V_r^\top \TZ f_\ell} / \lambda_i)_{\ell\in[m]}  =   \Gamma^f (\levec_i ^*V_r^\top)^\top / (\lambda_i \sqrt{n})\in\C^m.
\end{equation*}
\end{proof}

\begin{remark}\label{rem:krr_spectral}
The previous result can also be applied to KRR estimator $\EEstim_\reg$ since we can always take $r=n$, and take $U_r=I_n$ and $V_r = \Kx_{\reg}^{-1}$ to represent $\Kx_{\reg}^{-1} = U_r V_r^\top$. As a consequence, we can compute spectral decomposition of $\EEstim_\reg$ by solving a generalized eigenvalue problem 
\begin{equation}\label{eq:krr_gev}
\Kxy \revec_i = \lambda_i \Kx_\reg \revec_i, \quad i\in[n],      
\end{equation}
and setting $\levecs :=\revecs^{-*}$. This is possible when $\Kx_\reg^{-1}\Kxy$ is non-defective matrix, which is typically the case for kernel Gram matrices from real data.
\end{remark}

\section{Learning Bounds}\label{app:bounds}

\subsection{Uniform Bounds for i.i.d. Data}\label{app:bounds_unif_iid}
We first present a concentration inequality for bounded finite rank self-adjoint operators, which is a natural extension of \cite[Theorem 4]{MauPon13}, that dealt with positive operators.
\begin{proposition}
\label{prop:11bis}
Let $A_1,\dots,A_n$ be independent random operators of finite rank $\tau$ and $\|A_i\|\leq 1$, $i \in [n]$. Then 
\begin{equation}
\PP\left\{\left\|\sum_{i=1}^n A_i - \EE A_i\right\| > s\right\} \leq 8 (n\tau)^2 \exp \left\{\frac{-s^2}{36  \|\sum_i \EE (A_i^*A_i)^\frac{1}{2}\|+12s}\right\}.
\label{eq:gen1}
\end{equation}
\end{proposition}
\begin{proof}
Let $B_i =
       \left[                      
        \begin{array}{cc}
         {0}  & {A_i}  \\
          {A_i^*}   & {0}  
         \end{array}
      \right]  = P_i - N_i$, where
\begin{equation*}
P_i= \frac{1}{2}  \left[                      
        \begin{array}{cc}
          (A_iA_i^*)^{\frac{1}{2}} & A_i  \\
          A_i^*  &  (A_i^* A_i)^{\frac{1}{2}}
         \end{array}
      \right]  ~~~{\rm and~}~~~
N_i= \frac{1}{2}  \left[                      
        \begin{array}{cc}
          (A_i A_i^*)^{\frac{1}{2}} & -A_i  \\
          -A_i^*  &  (A_i^* A_i)^{\frac{1}{2}}
         \end{array}
      \right]. 
\end{equation*}
One verifies that the operators $P_i$ and $N_i$ are positive semi-definite, have the same rank as $A_i$, and $\|P_i\| = \|N_i\| = \|A_i\|$. 
Then
     \begin{eqnarray*}
\PP\left\{\left\|\sum_{i=1}^n A_i - \EE A_i\right\| > s\right\} & =&  \PP\left\{\left\|\sum_{i=1}^n B_i - \EE B_i\right\| > s\right\}  \\
& \leq &  \PP\left\{\left\|\sum_{i=1}^n P_i - \EE P_i \right\| + \left\|\sum_{i=1}^n N_i - \EE N_i \right\| > s \right\} \\
& \leq &  \PP\left\{\left\|\sum_{i=1}^n P_i - \EE P_i \right\| > \frac{s}{2} \right\} +\PP\left\{\left\|\sum_{i=1}^n N_i - \EE N_i \right\| > \frac{s}{2} \right\} \\
& \leq & 8 (n\tau)^2 \exp \left\{\frac{-s^2}{36 \max \{ \|\sum_i \EE P_i\|, \|\sum_i \EE N_i\|\} +12s}\right\}
\label{eq:rrr}
\end{eqnarray*} 
where the first inequality follows by triangle inequality, the second by the union bound and the last from~\cite[Thm.~7-(i)]{MauPon13}. The result follows by noting that $\|\sum_i\EE  P_i\| = \|\sum_i\EE  N_i\| \leq \|\sum_i\EE  (A_i^*A_i)^{\frac{1}{2}}\|$.
\end{proof}

A special case of the above proposition is Prop.~\ref{prop:11_main} which we restate here for the reader's convenience.

\propone*
\begin{proof}
We apply Prop.~\ref{prop:11bis} with $A_i =  \phi(x_i) \otimes \phi(y_i)$ and $\tau = 1$. We have
\begin{equation*}
P_i = \frac{1}{2}
       \left[                      
        \begin{array}{cc}
         \phi(x_i) \otimes \phi(x_i)  & \phi(x_i) \otimes \phi(y_i)  \\
          \phi(y_i) \otimes \phi(x_i) & \phi(y_i) \otimes \phi(y_i)         \end{array}
      \right]
      {~~\rm and~~} N_i =
      \frac{1}{2}
       \left[                      
        \begin{array}{cc}
        ~~~\phi(x_i) \otimes \phi(x_i)  & -\phi(x_i) \otimes \phi(y_i)  \\
          -\phi(y_i) \otimes \phi(x_i)   & ~~~\phi(y_i) \otimes \phi(y_i)  
         \end{array}\right].
 \end{equation*}
Since $(x_1,y_1),\dots,(x_n,y_n)$ are i.i.d. from $\rho$, for every $i \in [n]$
 \begin{equation}
\EE P_i = \
\frac{1}{2}
       \left[
        \begin{array}{cc}
         \Cx  & \Cxy  \\
          \Cyx   & \Cy 
         \end{array}
      \right]{~~\rm and~~}\EE N_i =\frac{1}{2}
       \left[                      
        \begin{array}{cc}
         ~~~\Cx  & -\Cxy  \\
          -\Cyx   & \Cy 
         \end{array}\right]
         \label{eq:expe}
 \end{equation}
Thus $\|\EE P_i \| = \|\EE N_i \|\leq \sqrt{\|C\| \|D\|} = \norm{\Cx}$, where the last equality is due to $\im$ being invariant measure and, hence, $\Cy = \Cx$. Then setting the r.h.s. of \eqref{eq:gen1} equal to $\delta$ and solving for $s$ gives Prop.~\ref{prop:11_main}.
\end{proof}

\thmmain*
\begin{proof}
Recalling the definition of the risk $\Risk(\Estim) =  \tr\big[ \Cy \big] + \tr  \big[GG^*C\big] - 2  \tr \big[G^*T\big]$,  and, analogously, empirical risk, a direct computation gives that 
\begin{eqnarray}
\Risk(G) - \widehat{\Risk}(G) & = & \tr\big( \Cy - \ECy \big)  + \tr  \big(GG^* (C - {\hat C})\big) - 2  \tr \big(G^*(T- {\hat T}) \big)  \nonumber \\ 
& \leq & \tr\big( \Cy - \ECy \big)+ \reg^2 \|\Cx - \ECx\|  + 2  \sqrt{r} \reg \|\Cxy -\ECxy\|, 
\label{jittu}
\end{eqnarray}
where we have used H\"older inequality in to obtain the last two terms in \eqref{jittu}.
First, we use Bernstein's inequality for bounded random variables \cite[Thm~ 2.8.4]{vershynin2018} to bound the first term in the r.h.s. of \eqref{jittu}, obtaining
\[
\tr(\Cy-\ECy)\leq \frac{\ln \frac{2}{\delta}}{3n} + \sqrt{\frac{2 \sigma^2 \ln \frac{2}{\delta}}{n}}.
\]
Then we use \cite[Theorem~7-(i)]{MauPon13} to bound the second term in the r.h.s. of \eqref{jittu}, and Prop.~\ref{prop:11_main} to bound the last term. The result then follows by a union bound.
\end{proof}

We expand some of the remarks stated after Thm.~\ref{thm:UB_main_text} in the main body of the paper.
First, note that when the measure $\im$ is not assumed to be invariant, we can cover the general CME case. In that case the term $\|C\|$ in the bound should be replaced by $\sqrt{\|C\|\|D\|}$, where, recall, $D$ is the covariance of the output. 
Second, using \cite[Cor.~3.1]{minsker2017} in place of Prop.~\ref{prop:11_main} one can derive a related bound which essentially replaces the term $\|C\|$ with $\| \EE AA^*\|$ where $A:=(\phi(x)\otimes \phi(y) - T)$.
This bound is more difficult to turn into a data dependent bound, but it allows for a more direct comparison to (potentially much larger) bounds without the rank constraint, where the quantity $\| \EE AA^*\|$ 
is replaced by the potentially much larger term 
$\tr \EE [\phi(x)\otimes \phi(x) \|\phi(y)\|^2 - T^*T]$.

Finally, Thm.~\ref{thm:UB_main_text} can be used to derive an excess risk bound in well specified case $\TZ_{\im} = \TS_\im\HKoop$. 
The analysis follows the pattern in \cite{luise2019}. We use the decomposition
\begin{equation}
\label{eq:dec}
\ExRisk(\EEstim) \leq 2 \sup_{\Estim \in \G_{r,\gamma}}  |\Risk(\Estim) - \widehat{\Risk}(\Estim)| + \ExRisk(\Estim_{r,\gamma})
\end{equation}
where $G_{r,\gamma}= {\rm argmin}_{\Estim \in \G_{r,\gamma}} \hnorm{\TS_\im(\HKoop-\Estim)}^2$. We bound the first term in the r.h.s. of \eqref{eq:dec} by Thm.~\ref{thm:UB_main_text}. The second term is the approximation error of $\HKoop$ in the class $\G_{r,\gamma}$. 
We next optimize over $\gamma$. A natural choice is $\gamma=\hnorm{[\![\HKoop]\!]_r}$ so that 
$G_{r,\gamma} = [\![\HKoop]\!]_r$, 
the truncated rank $r$ SVD of $\HKoop$. For this choice the approximation error 
$\Risk(\Estim_{\gamma,r})$ is $\epsilon_{r}:= \hnorm{\TS_\im(\HKoop - [\![\HKoop]\!]_r)}^2$. If $\HKoop$ has a fast decaying spectrum this error will be small for moderate sizes of $r$. In general since $\HKoop$ is Hilbert-Schmidt 
$\epsilon_{r}\rightarrow 0$ as $r\rightarrow \infty$. Replacing $\gamma=\hnorm{[\![\HKoop]\!]_r}$ in the uniform bound, then yields the excess risk bound (discarding $ O\big(1/n\big)$ terms and simplifying the constants)
\[
\ExRisk(\EEstim) \leq 
3 \hnorm{ [\![\HKoop]\!]_r}  \Big(6\sqrt{r}+ \hnorm{[\![\HKoop]\!]_r} \Big)  \sqrt{ \frac{\norm{\Cx} \ln \frac{24n^2}{\delta}}{n}}  + \sqrt{\frac{2\sigma^2 \ln \frac{6}{\delta}}{n}}  + \hnorm{\HKoop - [\![\HKoop]\!]_r}^2.
\]
This bound may be further optimized over $r$ if information on the spectrum decay of $\HKoop$ is available.

\subsection{Uniform Bounds for Data from a Trajectory}\label{app:bounds_unif_mix}

To prove Lem.~\ref{lem:blockprocess} we temporarily introduce extra notation. For a
set $I\subseteq $ $\mathbb{N}$ and a strictly stationary process ${\bf X}=(
X_{i})_{i\in \N} $ we let $\Sigma _{I}$ for the $\sigma $-algebra generated by $%
\left\{ X_{i}\right\} _{i\in I}$ and $\mu _{I}$ for the joint distribution
of $\left\{ X_{i}\right\} _{i\in I}$. Notice that $\mu _{I+i}=\mu _{I}$. In
this notation $\pi =\mu _{\left\{ 1\right\} }$ and $\rho _{\tau }=\mu
_{\left\{ 1,1+\tau \right\} }$. 

Then the definition of the mixing coefficients reads%
\[
\beta _{\mathbf{X}}\left( \tau \right) =\sup_{B\in \Sigma \otimes \Sigma
}\left\vert \mu _{\left\{ 1,1+\tau \right\} }\left( B\right) -\mu _{\left\{
1\right\} }\times \mu _{\left\{ 1\right\} }\left( B\right) \right\vert 
\]%
which by the Markov property is equivalent to%

\[
\beta _{\mathbf{X}}\left( \tau \right) =\sup_{B\in \Sigma ^{I}\otimes \Sigma
^{J}}\left\vert \mu _{I\cup J}\left( B\right) -\mu _{I}\times \mu _{J}\left(
B\right) \right\vert ,
\]%
where $I,J\subset \mathbb{N}$ with $j>i+\tau $ for all $i\in I$ and $j\in J$. 
The latter is the definition of the mixing coefficients for general
strictly stationary processes, for which we prove Lemma~\ref{lem:blockprocess}. We first need the following lemma.
\begin{lemma}
\label{Lemma betamix}Let $B\in \Sigma _{[1:m]}$. Then%
\[
\left\vert \mu _{[1:m]}\left( B\right) -\mu _{\left\{
1\right\} }^{m}\left( B\right) \right\vert \leq \left( m-1\right) \beta _{%
\mathbf{X}}\left( 1\right) .
\]
\end{lemma}
\begin{proof}
By stationarity, Fubini's Theorem and the definition of the mixing
coefficients, we have for $k \in [m]$, that%

\[
\left\vert \mu^{k-1} _{\left\{ 1\right\} } \times \mu _{[k:m]}\left( B\right) -\mu^{k-1} _{\left\{ 1\right\} }\times
\mu _{\left\{ 1\right\} }\times \mu _{[k+1:m]}\left(
B\right) \right\vert \leq \beta _{\mathbf{X}}\left( 1\right) .
\]%
Then, again with stationarity and a telescopic expansion,%
\begin{eqnarray*}
\left\vert \mu _{[1:m]}\left( B\right) -\mu _{\left\{
1\right\} }^{m}\left( B\right) \right\vert  
&=&\left\vert \sum_{k=1}^{m-1}\left( \mu^{k-1} _{\left\{ 1\right\} }\times \mu _{[k:m]}\left( B\right) -\mu^{k-1} _{\left\{
1\right\} } \times \mu _{\left\{ 1\right\} }\times \mu _{[
k+1:m]}\left( B\right) \right) \right\vert  \\
&\leq &\sum_{k=1}^{m-1}\left\vert \mu^{k-1} _{\left\{ 1\right\} }\times
\mu _{[k:m]}\left( B\right) -\mu^{k-1} _{\left\{ 1\right\}
}\times \mu _{\left\{ 1\right\} }\times \mu _{[k+1:m]}\left( B\right) \right\vert  \\
&\leq &\left( m-1\right) \beta _{\mathbf{X}}\left( 1\right) .
\end{eqnarray*}
\end{proof}

Now recall the definition of the blocked variables%
\[
Y_{j}=\sum_{i=2\left( j-1\right) \tau +1}^{\left( 2j-1\right) \tau }X_{i}%
\quad \quad  \text{ and }\quad Y_{j}^{\prime }=\sum_{i=\left( 2j-1\right) \tau +1}^{2j\tau
}X_{i},~~\text{ for }j\in \mathbb{N}.
\]%
Since the blocked variables are separated by $\tau $ we have $\beta _{\mathbf{Y}}(1) =\beta_{\mathbf{Y}^{\prime }}(1)=\beta _{\mathbf{X}}(\tau)$. 

\lemkey*
\begin{proof}
We can write%
\[
\left\Vert \sum_{i=1}^{n}X_{i}\right\Vert =\left\Vert
\sum_{j=1}^{m}Y_{j}+\sum_{j=1}^{m}Y_{j}^{\prime }\right\Vert \leq \left\Vert
\sum_{j=1}^{m}Y_{j}\right\Vert +\left\Vert \sum_{j=1}^{m}Y_{j}^{\prime
}\right\Vert.
\]%
Thus%
\[
\Pr \left\{ \left\Vert \sum_{i=1}^{n}X_{i}\right\Vert >s\right\}  \leq \Pr
\left\{ \left\Vert \sum_{j=1}^{m}Y_{j}\right\Vert +\left\Vert
\sum_{j=1}^{m}Y_{j}^{\prime }\right\Vert >s\right\}  \leq 2\Pr \left\{ \left\Vert \sum_{j=1}^{m}Y_{j}\right\Vert >\frac{s}{2}%
\right\} ,
\]
where the last inequality follows from identical distribution of $Y_{j}$ and 
$Y_{j+1}$. The conclusion then follows from applying Lem.~\ref{Lemma
betamix} to the event $B=\left\Vert \sum_{j=1}^{m}Y_{j}\right\Vert >\frac{s}{%
2}$. 
\end{proof}
Any available bound on the probability in the right hand side of Lem.~\ref{lem:blockprocess} can then be substituted to give a bound on the trajectory. 
To illustrate this we give a proof of Prop.~\ref{prop:11_main_bis} which we restate here for convenience.
\proponebis*
\begin{proof}
We use this Lem.~\ref{lem:blockprocess} 
with $X_{i}=\phi(x_{i})\otimes \phi(x_{i+1})- T$. We have
\[
\nonumber\mathbb{P}\left\{ 
\left\| {\hat T} - T \right\| >s \right\}  = \mathbb{P}\left\{ 
\left\| \sum_{i=1}^n X_i \right\| > ns\right\} \\
 \leq  2
\mathbb{P}\left\{ \left\| 
\sum_{j=1}^{m} Z_j \right\| >\frac{ns}{2}\right\} +2\left( m-1\right) \beta_{\mathbf{%
X}}\left( \tau-1 \right). 
\]
To bound the rightmost 
probability we then use Prop.~\ref{prop:11bis} with $A_i$ i.i.d. operator $\frac{1}{\tau} \sum_{i=1}^\tau \phi(x_i) \otimes \phi(x_{i+1})$ and $\EE (A_i^*A_i)^\frac{1}{2} \leq \|C\|$. We then solve for $\delta >( m-1) \beta ( \tau-1)$.
\end{proof}



\section{Experiments}\label{app:exp}
We developed a Python module implementing different algorithms to perform KOR
Both CPUs and GPUs are supported. Code and experiments can be found at \href{https://github.com/CSML-IIT-UCL/kooplearn}{https://github.com/CSML-IIT-UCL/kooplearn}. The experiments have been conducted on a workstation equipped with an Intel(R) Core\textsuperscript{TM} i9-9900X CPU @ 3.50GHz, 48GB of RAM and a NVIDIA GeForce RTX 2080 Ti GPU.
\subsection{Noisy Logistic Map}
We now show how the {\em trigonometric} noise introduced in~\cite{Ostruszka2000} allows the evaluation of the {\em true} invariant distribution, transition kernel and Koopman eigenvalues.

{\small
\begin{figure}[t!]
\begin{center}
\includegraphics[width=0.75\textwidth]{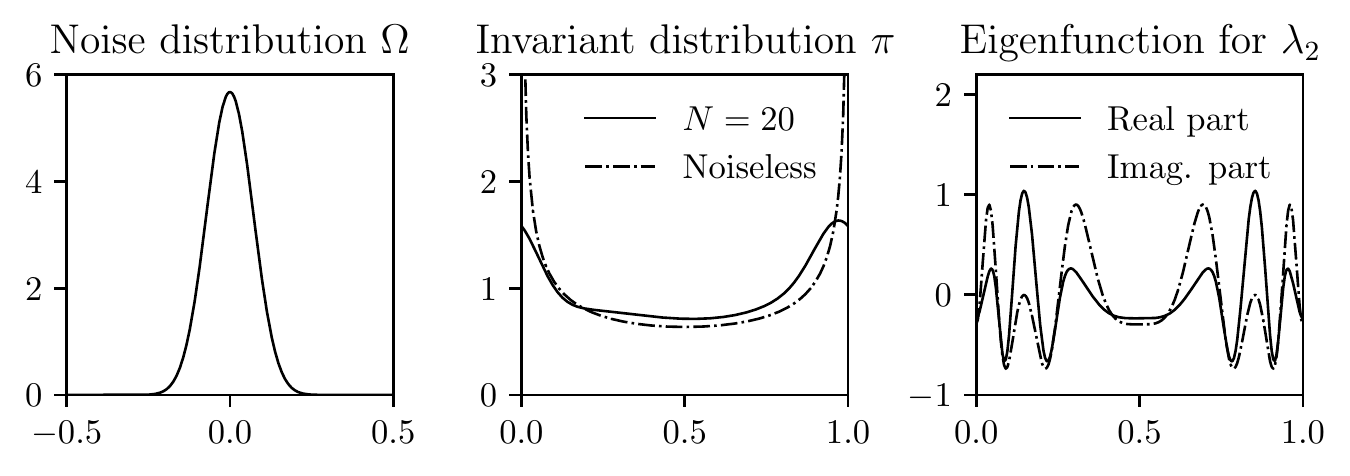}
\end{center}
    \caption{Noise distribution $\Omega$, invariant distribution $\pi$ and Koopman eigenfunction corresponding to the eigenvalue $\lambda_{2}$ for the case $N=20$. In the middle panel, the invariant distribution for the noiseless case ($N\to \infty$) is~\cite{Lasota1994} $\im_{{N \to \infty}}(dx) :=\left(\pi^{2}x(1-x)\right)^{-\nicefrac{1}{2}}dx$.}\label{fig:noisy_logistic_map_additonal_info}
\end{figure}
}
{\bf Trigonometric Noise.} We consider the {\em noisy logistic map}
\begin{equation*}
    x_{t + 1} = (4x_{t}(1 - x_{t}) + \xi_{t}) \mod 1 = (F(x_{t}) + \xi_{t}) \mod 1
\end{equation*}
over the state space $\X = [0 , 1]$. We have defined the logistic map $F(x) := 4x(1-x)$ for later convenience. Here, $\xi_{t}$ is i.i.d. additive noise with law ($N$ being an {\it even} integer) given by
\begin{equation*}
    \Omega(d\xi) := C_{N}\cos^{N}(\pi \xi)d\xi \qquad \xi \in [-0.5, 0.5].
\end{equation*}
The normalization constant is given by $C_{N}:= \pi/{\rm B}\left( \frac{N + 1}{2}, \frac{1}{2}\right)$, where ${\rm B}(\cdot, \cdot)$ is Euler's beta function. The noise is additive and as noted in Example 1 of Sec.~\ref{sec:koopman_theory}, the transition kernel is
\begin{equation}\label{eq:trigonometric_noise_transition_kernel}
    \transitionkernel(x, dy) = \Omega(dy - F(x)) = C_{N}\cos^{N}\left(\pi y - \pi F(x)\right)dy.
\end{equation}
We now show that the transition kernel~\eqref{eq:trigonometric_noise_transition_kernel} is {\em separable}. Indeed, for $i \in [0{:}N]$ let us define the functions
\begin{equation*}
        \beta_{i}(x) := \sqrt{C_{N}\binom{N}{i}} \cos^{i}(\pi  x)\sin^{N - i}(\pi x), \quad \quad {\rm and}\quad
        \alpha_{i}(x) := \left(\beta_{i} \circ F \right)(x).
\end{equation*}
By a simple application of the binomial theorem one has that (with a slight abuse of notation)
\begin{equation*}
    \transitionkernel(x, y) = \sum_{i = 0}^{N} \alpha_{i}(x)\beta_{i}(y),
\end{equation*}
implying that the transition kernel is separable and of finite rank $N + 1$. Therefore, the Koopman operator $\Koop$ is compact operator of  a finite rank operator at most $N+1$. Moreover, with the proper choice of the kernel $k$, we have that $\alpha_i\in\RKHS$ for all $i\in[0{:}N]$, implying that the Koopman operator regression problem is well-specified. 

Further, the eigenvalue equation for the Koopman operator requires to find $h:[0,1]\to[0,1]$ and $\lambda \in \C$ satisfying
\begin{equation}\label{eq:trigonometric_noise_eigenvalue_equation}
    \lambda h(x) = \int_{0}^{1} h(y)p(x, y)dy \qquad \text{for all } x \in[0,1].
\end{equation}
The solution of this {\em homogeneous Fredholm integral equation of the second kind}~\eqref{eq:trigonometric_noise_eigenvalue_equation} is easily obtained since the transition kernel is separable (see e.g. Section 23.4 of~\cite{Riley2006}). Indeed, let $P$ be the $(N + 1)\times(N + 1)$ matrix whose elements are $P_{ij} := \int_{0}^{1}\beta_{i}(x)\alpha_{j}(x)dx$. For any $\lambda$ eigenvalue of $P$ with corresponding eigenvector $(c_{i})_{i = 0}^{N}$, the function $h(x) := \sum_{i=0}^{N}\alpha_{i}(x)c_{i}$ 
is an eigenfunction of the Koopman operator with eigenvalue $\lambda$. With a similar argument, let $(d_{i})_{i = 0}^{N}$ be the eigenvector of $P^{T}$ corresponding to the eigenvalue $\lambda = 1$. The invariant distribution (up to a normalization constant) $\im$ is given by
\begin{equation*}
    \im(x)dx = \left(\sum_{i=0}^{N}\beta_{i}(x)d_{i}\right)dx.
\end{equation*}

Every result presented in the main text concerned the case $N=20$. In Fig.~\ref{fig:noisy_logistic_map_additonal_info} we show the noise distribution $\Omega$, invariant distribution $\pi$ and Koopman eigenfunction corresponding to the eigenvalue $\lambda_{2}$ for the case $N=20$.

\subsection{Additional experiment: the Lorenz63 Dynamical System}
The Lorenz63 system is given by the solution the differential equation $$\frac{d\bm{x}}{dt} = \begin{pmatrix}\sigma(x_{2} - x_{1}) \\ x_{1}(\mu - x_{3}) \\ x_{1}x_{2} - \beta x_{3} \end{pmatrix}.$$ In our experiments we have used the standard parameters $\sigma = 10$, $\mu = 28$ and $\beta = 8/3$. The solution to the ODE was obtained using the explicit Runge-Kutta method of order 5(4) as implemented by the function \verb|solve_ivp| of the Python library \verb|Scipy|~\cite{Virtanen2020}. We discarted any data before $t=100$ to give time to the solution to converge to the stable attractor~\cite{Tucker1999} and then sampled a data point every $\Delta t = 0.1$ (in natural time units). The Lorenz63 attractor is also known to be mixing~\cite{Luzzatto2005}.

\subsection{Additional details on the numerical verification of the uniform bounds}
In this section we discuss how we have obtained the results presented in Fig.~\ref{fig:uniform_bound_logistic} of the main main text for the logistic map and in Fig.~\ref{fig:uniform_bound_lorenz63} for the Lorenz63 dynamical system.

As remarked in the main text, the proposed RRR estimator and the classical PCR estimator satisfy the same uniform bound. However, the empirical risk may be (possibly much) smaller for the RRR estimator and hence preferable. To this end, we evaluated, as a function of the number of training points, the empirical risk of PCR and RRR estimators under the same HS-norm constraint, needed to satisfy the assumption of Theorem~\ref{thm:UB_main_text}.

To achieve the same HS norm for both estimators we first trained the PCR estimator and computed its HS norm. We then adjusted the Tikhonov regularization for the RRR estimator to a value yielding the same HS norm of the PCR estimator. We remark that this procedure in general does not yield the best (w.r.t. the regularization parameter) RRR estimator, but allows us to compare PCR and RRR within the Ivanov setting considered in Theorem~\ref{thm:UB_main_text}. Moreover, we have also verified that the upper bound on the scaling $\approx n^{-1/2}$ derived in Theorem~\ref{thm:UB_main_text} empirically holds. 

In both Logistic map and Lorenz63 each experiment was independently repeated 100 times, the number of test points is $5\times 10^4$ and RRR consistently attains smaller empirical risk than PCR. 

{\small
\begin{figure}[t!]
\begin{center}
\includegraphics[width=0.7\textwidth]{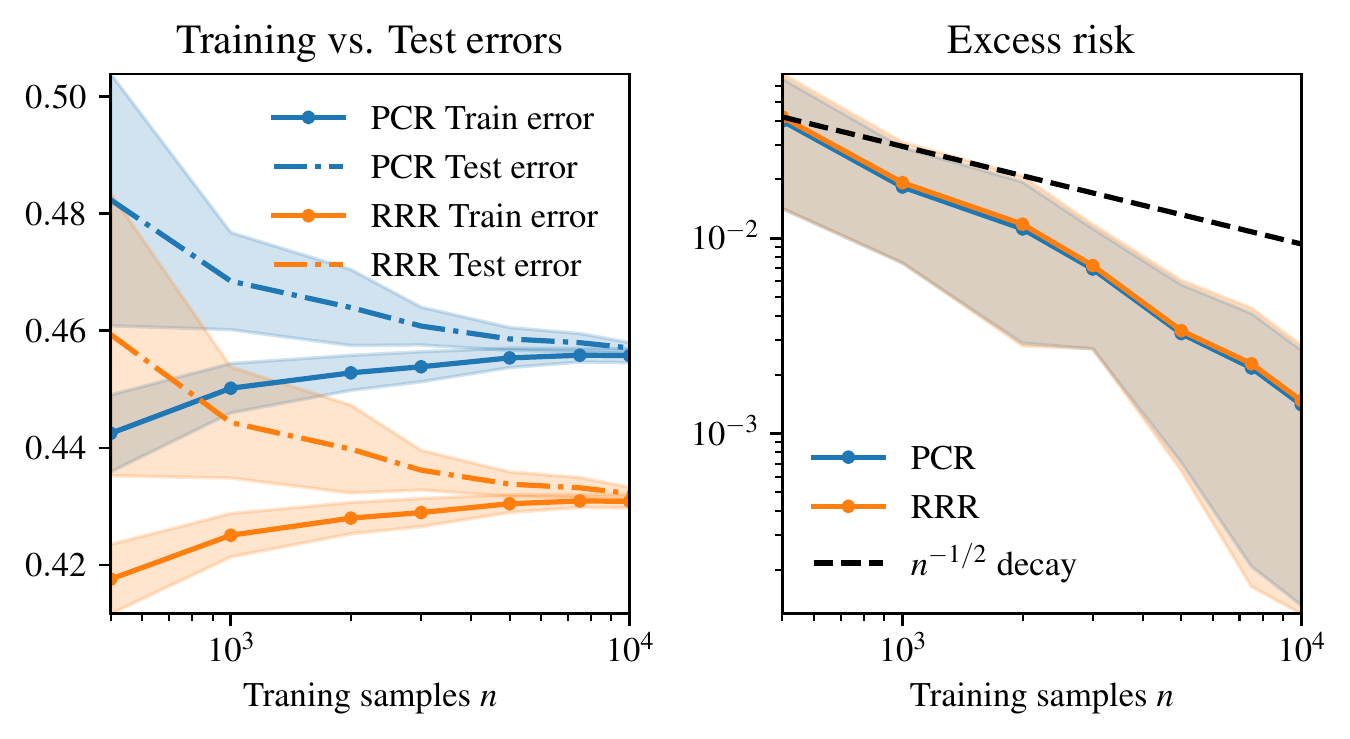}
\end{center}
\caption{Numerical verification of the uniform bound presented in Theorem~\ref{thm:UB_main_text} for the Lorenz63. Left panel: the training and test risk for RRR are consistently than PCR. Right panel: the deviation between training and test risk decreases faster than $n^{-1/2}$ as a function of the number of training samples.}  \label{fig:uniform_bound_lorenz63}
\end{figure}
}

\begin{table}[t!]
    \small
    \caption{Comparison of the estimators trained for the Beijing air quality experiment.
    }\label{tab:Beijing_training_errors}
    \centering
    \begin{tabular}{r|cc}
        \toprule
        Estimator & Training error & Test error  \\ 
        \midrule
        PCR & $  0.5809 $ & $ 0.5923 $  \\
        RRR & $ \bm{0.5780}$ & $\bm{0.5899}$  \\
        \bottomrule
    \end{tabular}
\end{table}

\begin{table}[ht]
    \small
    \caption{Delay between wind speed peaks and PM2.5 concentration peaks. Positive values correspond to peaks in wind speed occurring {\em after} peaks in PM2.5 concentration. Coupled modes correspond to complex conjugate pairs. Modes $1,6,9$ and $10$ correspond to real eigenvalues and delays can't be evaluated.
    }\label{tab:Beijing_delay_table}
    \centering
    \begin{tabular}{r|ccccccc}
        \toprule
        Station & Mode 1 & Modes 2-3 & Modes 4-5 & Mode 6 & Modes 7-8 & Mode 9 & Mode 10 \\ 
        \midrule
        Guanyuan &  - & 1.92 hrs. & 2.74 hrs. &  - & 1.69 hrs. &  - &  - \\
        Aotizhongxin &  - & 1.89 hrs. & 2.61 hrs. &  - & 1.64 hrs. &  - &  - \\
        Wanshouxigong &  - & 2.01 hrs. & 2.82 hrs. &  - & 1.87 hrs. &  - &  - \\
        Tiantan &  - & 2.0 hrs. & 2.92 hrs. &  - & 1.83 hrs. &  - &  - \\
        Nongzhanguan &  - & 2.01 hrs. & 2.96 hrs. &  - & 1.84 hrs. &  - &  - \\
        Gucheng &  - & 2.06 hrs. & 2.54 hrs. &  - & 1.77 hrs. &  - &  - \\
        Wanliu &  - & 2.01 hrs. & 3.08 hrs. &  - & 1.66 hrs. &  - &  - \\
        Changping &  - & 2.04 hrs. & 2.79 hrs. &  - & 1.51 hrs. &  - &  - \\
        Dingling &  - & 2.0 hrs. & 2.67 hrs. &  - & 1.31 hrs. &  - &  - \\
        Huairou &  - & 2.02 hrs. & 2.31 hrs. &  - & 1.45 hrs. &  - &  - \\
        Shunyi &  - & 1.93 hrs. & 2.56 hrs. &  - & 1.42 hrs. &  - &  - \\
        Dongsi &  - & 1.97 hrs. & 2.76 hrs. &  - & 1.8 hrs. &  - &  - \\
        \bottomrule
    \end{tabular}
\end{table}

{\small
\begin{figure}[t!]
\centering
\includegraphics[width=0.6\textwidth]{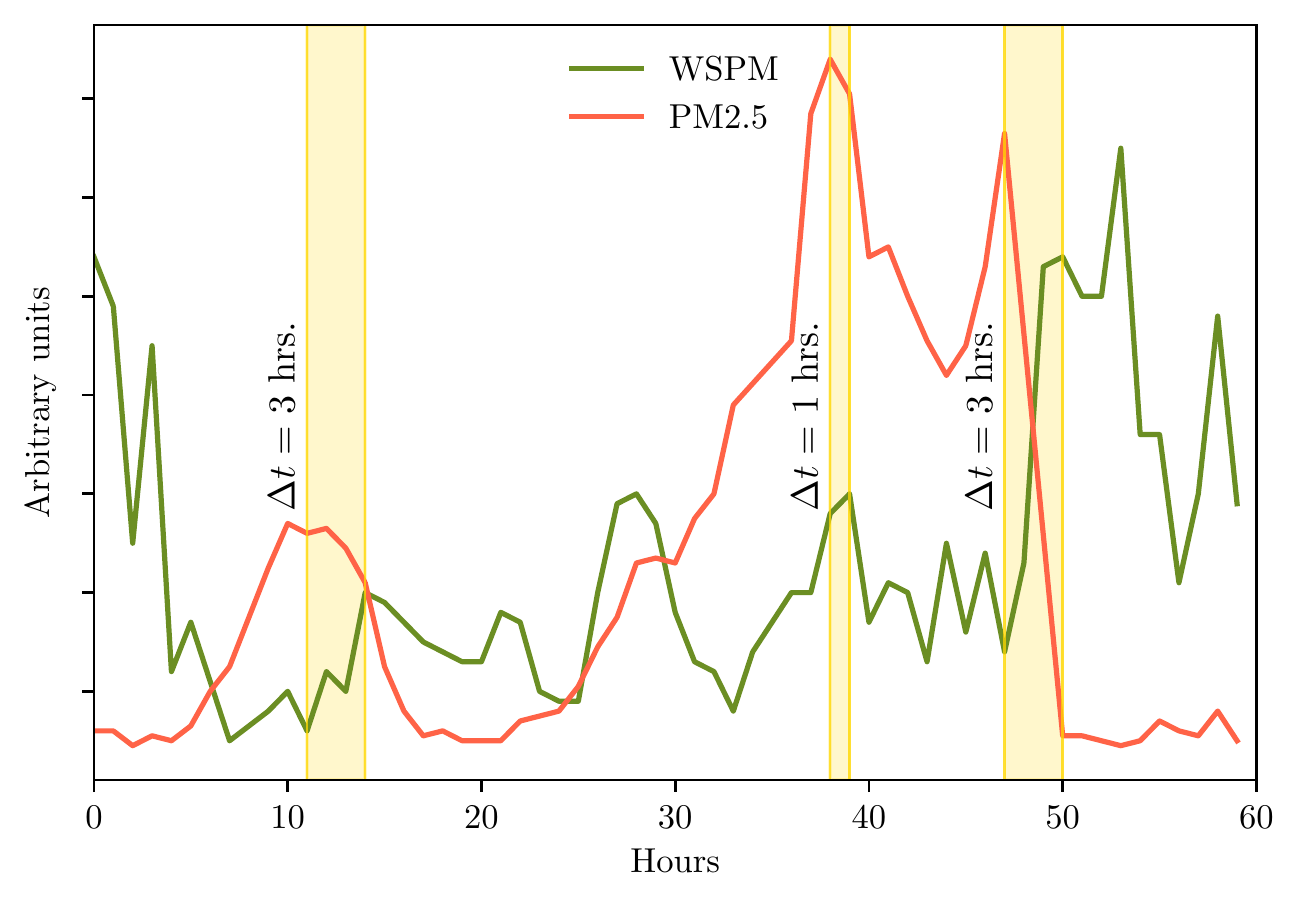}
    \caption{60 hours of data collected in the Gucheng station. Three peaks of PM2.5 concentration followed by peaks in wind speed. We have annotated the delay in hours between the two peaks.}\label{fig:Beijing_mode_delay} 
\end{figure}
}
\newpage
\subsection{Alanine dipeptide: additional plots}\label{app:ala2}

{\small
\begin{figure}[h!]
\begin{center}
\includegraphics[width=0.8\textwidth]{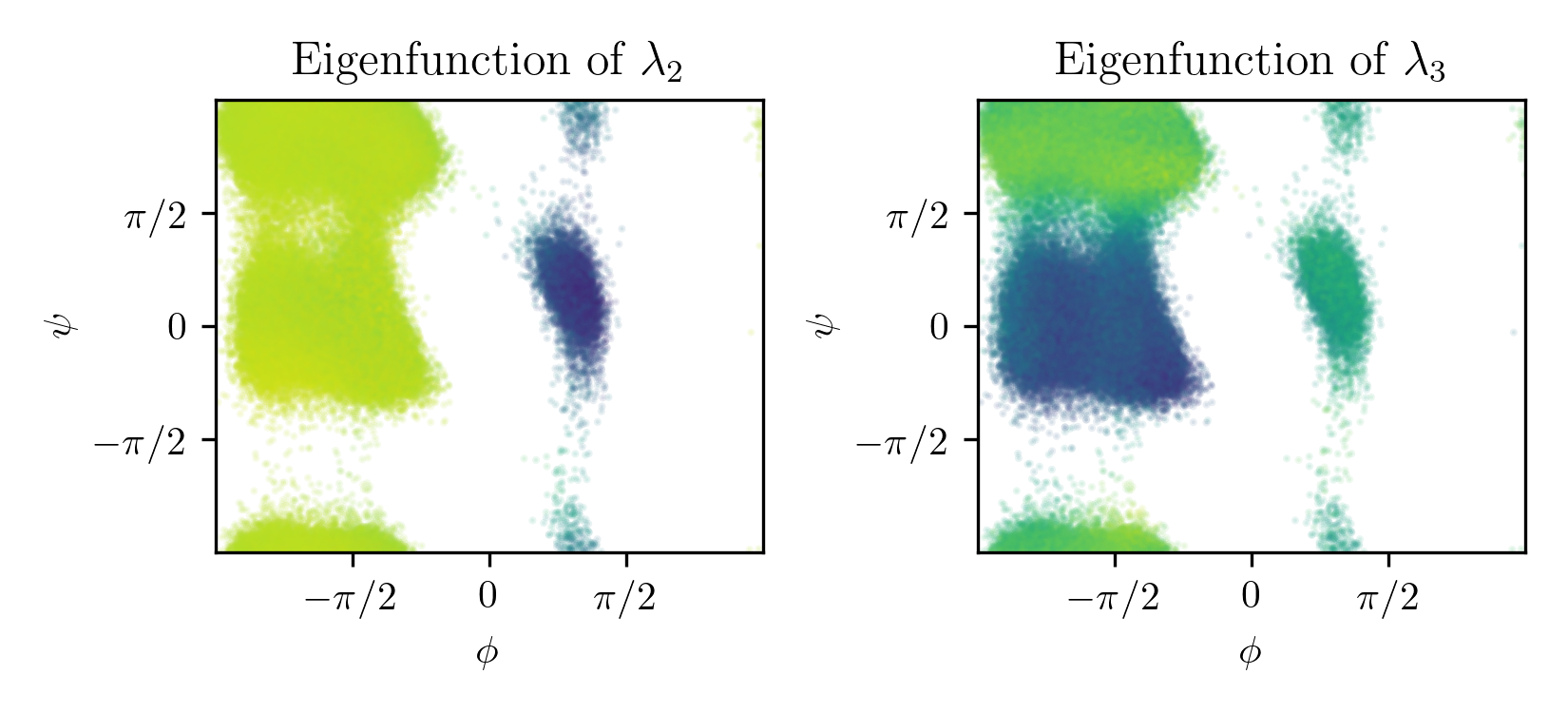}
\end{center}
\caption{Estimated eigenfunctions in the dihedral angle space of Alanine dipeptide. Each point correspond to a data point in the trajectory. The color encodes the value of the eigenfunction.}  \label{fig:ala2_lead_eigenfunctions}
\end{figure}
}

\subsection{Additional experiment: Beijing Air Quality Dataset}
This dataset~\cite{Zhang2017} consists of hourly measurements of six different air-pollutants along with relevant meteorological variables. Measurements were collected at twelve air-quality monitoring sites in Beijing, from March 1, 2013 to February 28, 2017. The analysis in~\cite{Zhang2017} showed that the presence of Particolate Matter smaller than $2.5~\mu m$ (PM2.5) is highly correlated to meteorological variables, like humidity and low Wind Speeds (WSPM). In this experiment we show how the modal decomposition of the Koopman operator can enrich the analysis in \cite{Zhang2017} with dynamical insights. Following this work, we analyse each season of the year separately for better meteorological homogeneity. 

We report data for RRR and PCR estimators ($r = 10$) over 7000 out of the 8564 hourly data points collected in winter using an Exponential kernel. Training and test errors are summarised in Tab.~\ref{tab:Beijing_training_errors}, RRR achieving slightly smaller test and training errors. The optimal regularization parameter for RRR was chosen by grid search, splitting the data via the \texttt{TimeSeriesSplit} as implemented in the \texttt{scikit-learn}~\cite{scikit-learn} package. Regularization $\reg = 10^{-4}$ turned out to be optimal.

As showed in~\cite{Proctor2015}, analysing the phase difference of modes corresponding to different observables allows us to infer whether variations of one observable are followed or anticipated by variations of another. Indeed, the modes corresponding to wind speed (WSPM) and PM2.5 concentration reported in Table~\ref{tab:Beijing_delay_table}, consistently point out that peaks in PM2.5 are {\it followed} by peaks in WSPM with a delay of $\approx 2$ hours. This is reasonable as high wind speeds favour the dispersion of PM2.5, and as wind ramps up toward a peak, PM2.5 concentration is reduced. Reference~\cite{Walcek2002}, indeed, argue that pollution concentration is fully readjusted on the basis of wind conditions already after $4$ hours. 

As a a final illustrative example, in Fig.~\ref{fig:Beijing_mode_delay} we show an excerpt spanning 60 hours of data collected in the Gucheng station. We have manually identified three peaks in the PM2.5 concentration followed only $\approx 2/3$ hours later by peaks in the wind speed.

\subsection{
Koopman Operator Regression 
with Deep Learning Embeddings}
We have used a Linear kernel $\ell^{-1}\left\langle x,x^{\prime}\right\rangle$ and a Gaussian kernel with length scale $\ell$. Here $\ell = 28\times 28 = 784$ is the number of pixels in each image. The regularization parameter was chosen by grid search, splitting the data via the \texttt{TimeSeriesSplit} as implemented in the \texttt{scikit-learn}~\cite{scikit-learn} package. The optimal regularization parameters are, respectively $\reg_{{\rm lin}} {=} 48.33$ and $\reg_{{\rm gauss}} {=} 7.85\cdot 10^{-3}$. 

The CNN kernel is $\left\langle \phi_{{\bm \theta}}(x), \phi_{{\bm \theta}}(x^{\prime})\right\rangle$, where the architecture of the network is given by $\phi_{{\bm \theta}} := {\rm Conv2d}(1,16;5)\rightarrow{\rm ReLU}\rightarrow{\rm MaxPool}(2)\rightarrow{\rm Conv2d}(16,32;5)\rightarrow{\rm ReLU}\rightarrow{\rm MaxPool}(2)\rightarrow{\rm Dense}(1568, 10)$. Here, the arguments of the convolutional layers are Conv2d(\verb|in_channels|, \verb|out_channels|; \verb|kernel_size|). The Tikhonov regularization parameter for the CNN kernel is $\reg_{{\rm CNN}} = 10^{-4}$. The network $\phi_{{\bm \theta}}$ has been pre-trained as a digit classifier using the cross entropy loss function. Training was performed with the Adam optimizer (learning rate $= 0.01$) for 20 epochs (batch size = 100). The training dataset corresponds to the {\em same} 1000 images used to train the Koopman estimators.

In Fig.~\ref{fig:CNN_kernel_full_seeds} we compare Linear, Gaussian, and CNN kernels for different initial seeds. As it can be noticed the CNN kernel remains strong across the board, while the forecasting ability of the linear and Gaussian kernels quickly deteriorate as $t$ increases.

\begin{figure}[th]
\caption{Comparison of different kernels in the generation of a series of digits. Starting from a seed image, the next ones are obtained by iteratively using a rank-10 RRR Koopman operator estimator. As in the main text, the first row of each panel corresponds to the Linear kernel, second to Gaussian kernel and last row to CNN kernel.}\label{fig:CNN_kernel_full_seeds}
\begin{tabular}{cc}
  \includegraphics[width=0.47\textwidth]{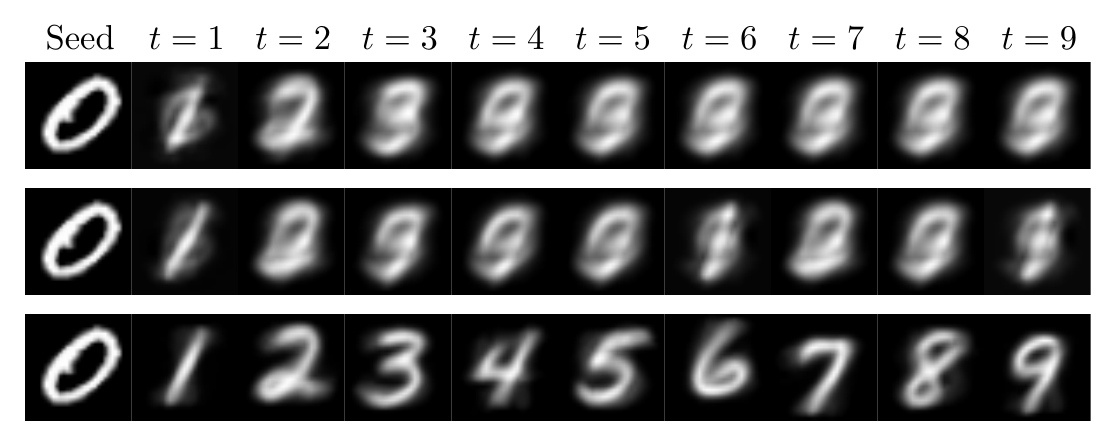} &   \includegraphics[width=0.47\textwidth]{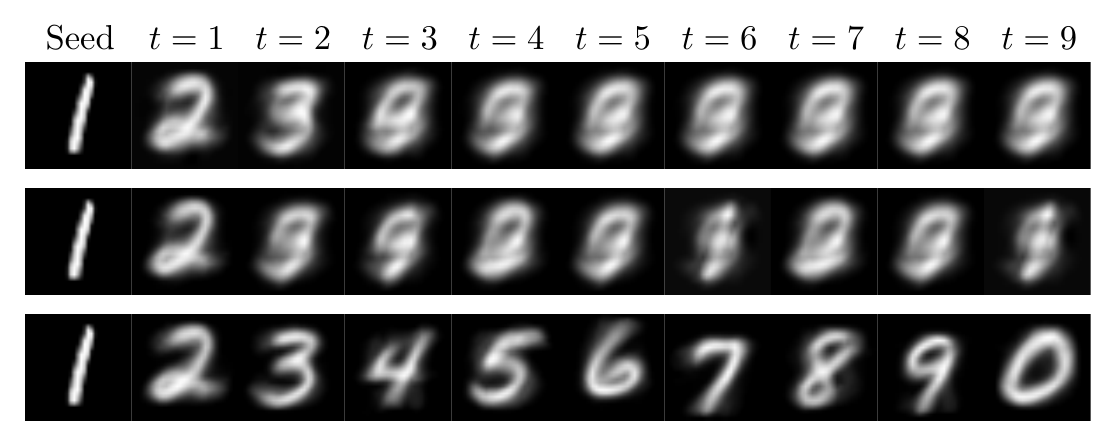} \\
 \includegraphics[width=0.47\textwidth]{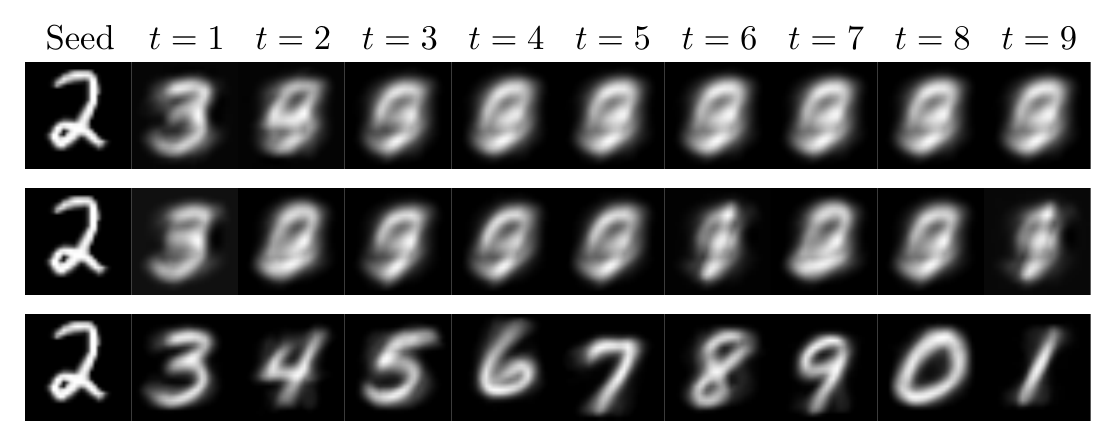} &   \includegraphics[width=0.47\textwidth]{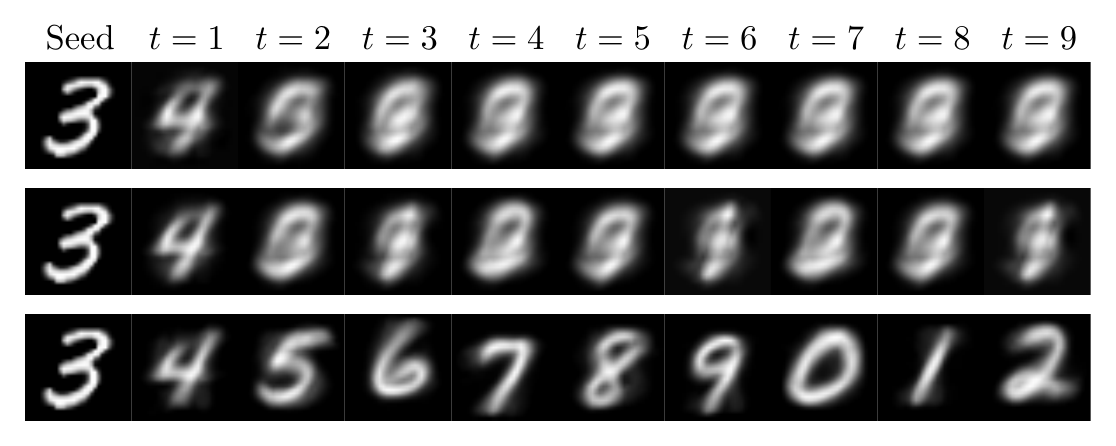} \\
 \includegraphics[width=0.47\textwidth]{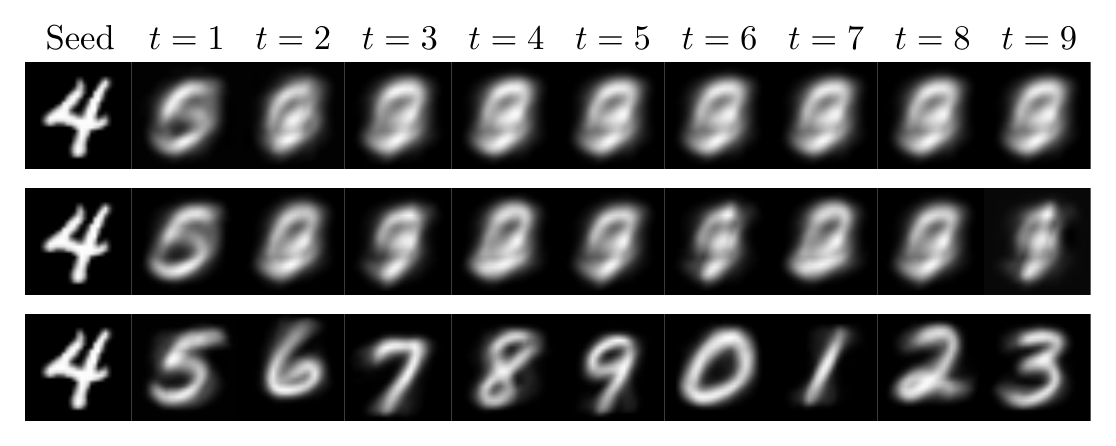} &   \includegraphics[width=0.47\textwidth]{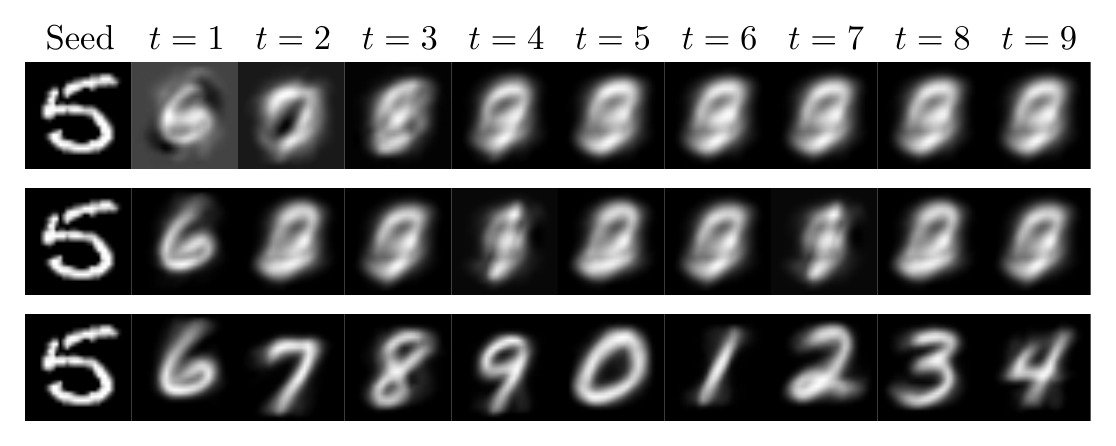} \\
 \includegraphics[width=0.47\textwidth]{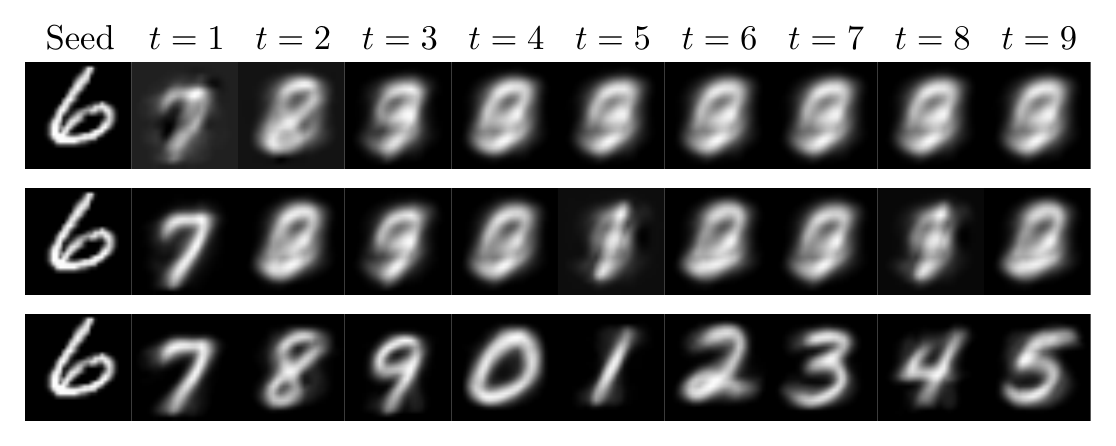} &   \includegraphics[width=0.47\textwidth]{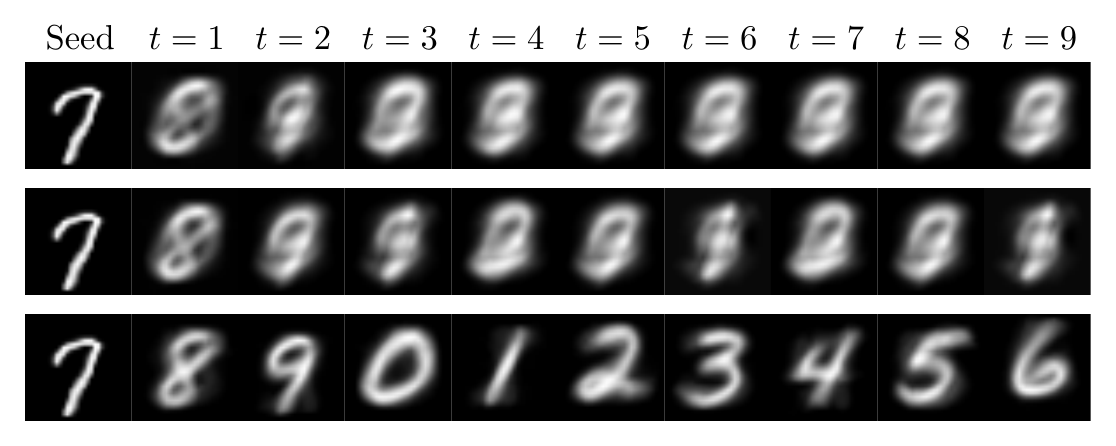} \\
 \includegraphics[width=0.47\textwidth]{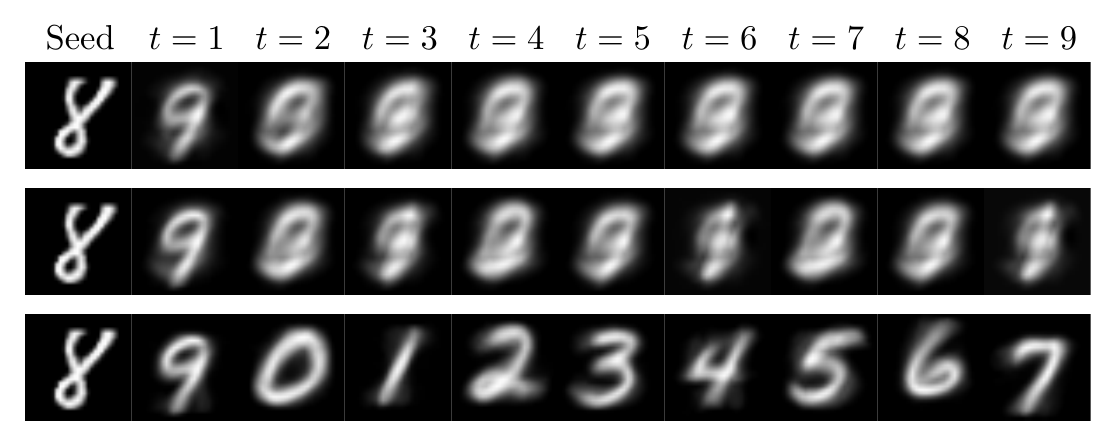} &   \includegraphics[width=0.47\textwidth]{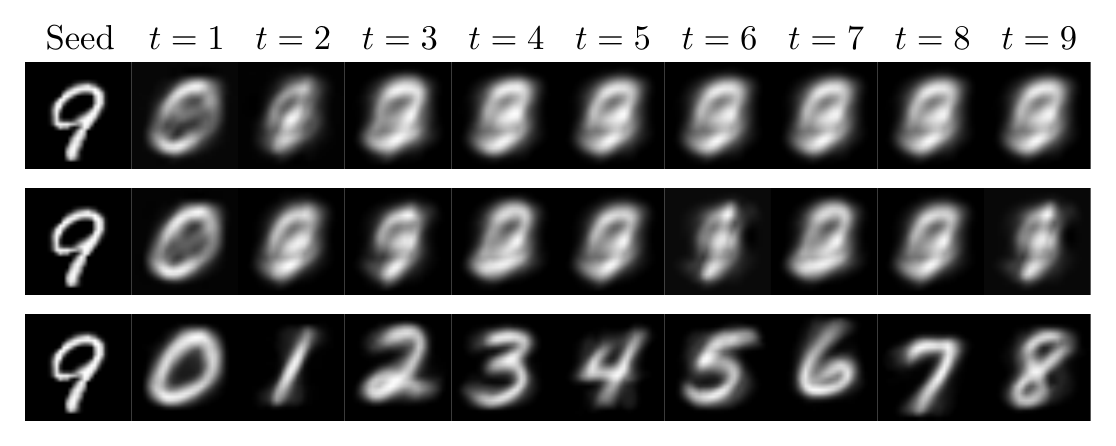} \\
\end{tabular}
\end{figure}



\end{document}